\documentclass{article} % For LaTeX2e

\usepackage[table,dvipsnames]{xcolor}

\usepackage{arxivtemplate/arxiv,times}

\usepackage{tikz}
\usepackage{comment}
\usepackage{amsmath,amssymb,amsthm} % define this before the line numbering.
\usepackage{color}

\usepackage{epsfig}
\usepackage{graphicx}

\usepackage{wrapfig}
\usepackage{xspace}
\usepackage{bbold}
\usepackage[nice]{nicefrac}
\usepackage{mathtools}
\usepackage{booktabs}
\usepackage{fontawesome}
\usepackage[font={footnotesize}]{caption}
\usepackage{subcaption}
\usepackage{stfloats}
\usepackage{pbox}
\usepackage{multirow}
\usepackage{microtype}

\definecolor{fbpurple3}{HTML}{f0ebf5}
\definecolor{fbteal2}{HTML}{199696}
\definecolor{fborange2}{HTML}{f06919}
\definecolor{fbApp}{HTML}{dee3e9}
\definecolor{fborange3}{HTML}{ffefe1}
\definecolor{fbApp}{HTML}{c8e7fa}

\newcommand{\norm}[1]{\left\lVert#1\right\rVert}
\newcommand{\defeq}{\coloneqq}

\newtheorem{proposition}{Proposition}
\newtheorem{assumption}{Assumption}

\usepackage{hyperref}
\usepackage{url}

\usepackage{amsmath,amsfonts,bm}

\def\eqref#1{equation~\ref{#1}}
\def\1{\bm{1}}

\DeclareMathAlphabet{\mathsfit}{\encodingdefault}{\sfdefault}{m}{sl}
\SetMathAlphabet{\mathsfit}{bold}{\encodingdefault}{\sfdefault}{bx}{n}

\newcommand{\R}{\mathbb{R}}

\title{Masked Siamese Networks\\ for Label-Efficient Learning}

\author{\bf Mahmoud Assran$^{1,2,5}$\thanks{correspondence to massran@fb.com}\quad Mathilde Caron$^{1,3}$\quad Ishan Misra$^{1}$\quad Piotr Bojanowski$^{1}$ \\ 
\bf Florian Bordes$^{1,2,4}$\quad Pascal Vincent$^{1,2}$\quad Armand Joulin$^{1}$\quad Michael Rabbat$^{1,2}$\quad
\bf Nicolas Ballas$^{1}$\\[2mm]
$^{1}$Facebook AI Research\\
$^{2}$Mila -- Quebec AI Institute \\
$^{3}$Inria, Univ.~Grenoble Alpes \\
$^{4}$Universite de Montreal, DIRO \\
$^{5}$McGill University, Dept. of Electrical and Computer Engineering
}

\begin{document}

\maketitle

\begin{abstract}
We propose Masked Siamese Networks (MSN), a self-supervised learning framework for learning image representations.
Our approach matches the representation of an image view containing randomly masked patches to the representation of the original unmasked image.
This self-supervised pre-training strategy is particularly scalable when applied to Vision Transformers since only the unmasked patches are processed by the network.
As a result, MSNs improve the scalability of joint-embedding architectures, while producing representations of a high semantic level that perform competitively on low-shot image classification.
For instance, on ImageNet-1K, with only 5,000 annotated images, our base MSN model achieves 72.4\% top-1 accuracy, and with 1\% of ImageNet-1K labels, we achieve 75.7\% top-1 accuracy, setting a new state-of-the-art for self-supervised learning on this benchmark. Our code is publicly available at \href{https://github.com/facebookresearch/msn}{https://github.com/facebookresearch/msn}.
\end{abstract}

\section{Introduction}
Self-Supervised Learning (SSL) has emerged as an effective strategy for unsupervised learning of image representations, eliminating the need to manually annotate vast quantities of data.
By training large models on unlabeled data, SSL aims to learn representations that can be effectively applied to a downstream prediction task with few labels~\citep{chen2020exploring}.

One of the core ideas of SSL is to remove a portion of the input and learn to predict the removed content~\citep{pathak2016context}.
Auto-regressive models and denoising auto-encoders instantiate this principle in vision by predicting the missing parts at the pixel or token level~\citep{chen2020generative,vincent2010stacked,he2021masked,bao2021beit,baevski2022data2vec}.
Masked auto-encoders in particular, which learn representations by reconstructing randomly masked patches from an input, have been successfully applied in vision~\citep{he2021masked,xie2021simmim,wei2021masked,bao2021beit}.
However, optimizing a reconstruction loss requires modelling low-level image details that are not necessary for classification tasks involving semantic abstraction. 
Thus, the resulting representations often need to be fine-tuned for semantic recognition tasks which can lead to overfitting in low-shot settings.
Nevertheless, masked auto-encoders have enabled the training of large-scale models and demonstrated state-of-the-art performance when fine-tuning on large labeled datasets, with millions of labels~\citep{bao2021beit,he2021masked,xie2021simmim,baevski2022data2vec}.

Joint-embedding architectures, on the other hand, avoid reconstruction. Approaches such as Siamese Networks~\citep{he2019moco,caron2020unsupervised,chen2020exploring,grill2020bootstrap,caron2021emerging,zbontar2021barlow,bardes2021vicreg} learn a representation by training an encoder network to produce similar embeddings for two different views of the same image~\citep{bromley1993signature,exemplarConvNet2014}.
Here the views are typically constructed by applying different image transforms --- such as random scaling, cropping, and color jitter --- to the input~\citep{wu2018unsupervised,misra2020self}.
The inductive bias introduced by this invariance-based pre-training typically produces strong off-the-shelf representations of a high semantic level~\citep{caron2021emerging} but often disregards rich local structure that can be helpful to model.

\begin{figure}[t]
    \centering
    \begin{subfigure}[t]{0.48\linewidth}
        \includegraphics[width=\linewidth]{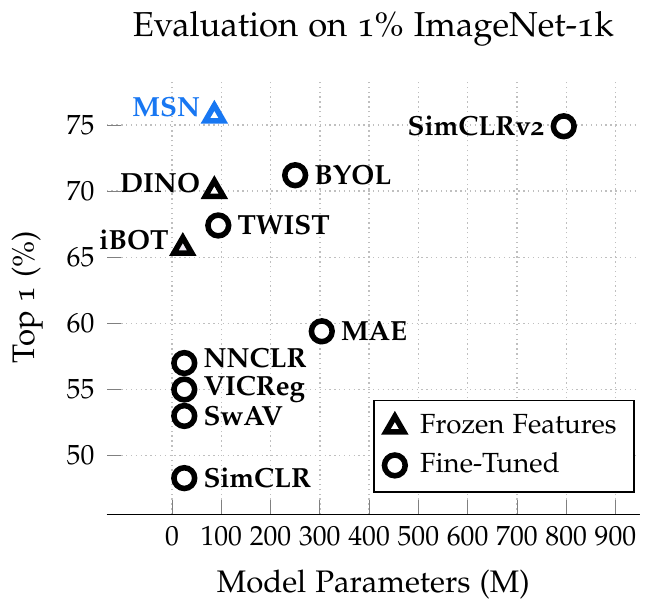}
        \caption{Evaluation using 1\% of ImageNet-1K labels ($\sim$13 imgs/class). Evaluation with {\it Frozen Features} corresponds to freezing the weights and training a logistic regression classifier with the available labeled samples. Evaluation with {\it Fine-Tuning} corresponds to adding a linear head and fine-tuning the model+head, end-to-end.}
        \label{fig:1percent}
    \end{subfigure}\hfill
    \begin{subfigure}[t]{0.48\linewidth}
        \includegraphics[width=\linewidth]{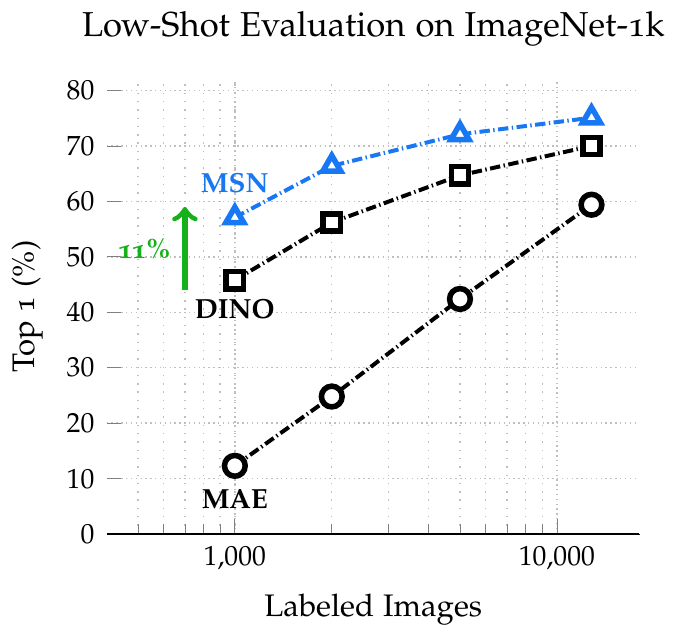}
        \caption{Low-shot evaluation comparing MSN (ViT-L/7) to the best publicly available models in low-shot classification for DINO (ViT-B/8) and MAE (ViT-L/16). MSN and DINO use a linear probe, whereas MAE uses partial fine-tuning, where the last block of the pre-trained model along with a linear head are adapted.}
        \label{fig:lowshot_sweep}
    \end{subfigure}
    \caption{{\bf Low-shot Evaluation of self-supervised models, pre-trained on ImageNet-1K.} (Left) MSN surpasses the previous 800M parameter state-of-the-art, while using a model that is $10\times$ smaller. (Right) MSN achieves good classification performance using less labels than current mask-based auto-encoders.}
    \label{fig:lowshot}
\end{figure}

In this work, we propose Masked Siamese Networks (MSNs), a self-supervised learning framework that leverages the idea of mask-denoising while avoiding pixel and token-level reconstruction.
Figure~\ref{fig:msn} shows a schematic of the method.
Given two views of an image, MSN randomly masks patches from one view while leaving the other view unchanged.
The objective is to train a neural network encoder, parametrized with a vision transformer (ViT)~\citep{dosovitskiy2020image}, to output similar embeddings for the two views.
In this procedure, MSN does not predict the masked patches at the input level, but rather performs the denoising step implicitly at the representation level by ensuring that the representation of the masked input matches the representation of the unmasked one.
Figure~\ref{fig:msn_sampling} qualitatively demonstrates the effectiveness of the MSN denoising process. %Since there is no decoder in the MSN framework, we visualize the image representations (output corresponding to {\tt[CLS]} token) by following the RCDM framework of~\citet{bordes2021high} (cf.~Appendix~\ref{apndx:qualitative}).
\begin{figure}[h]
    \centering
    \includegraphics[width=0.95\linewidth]{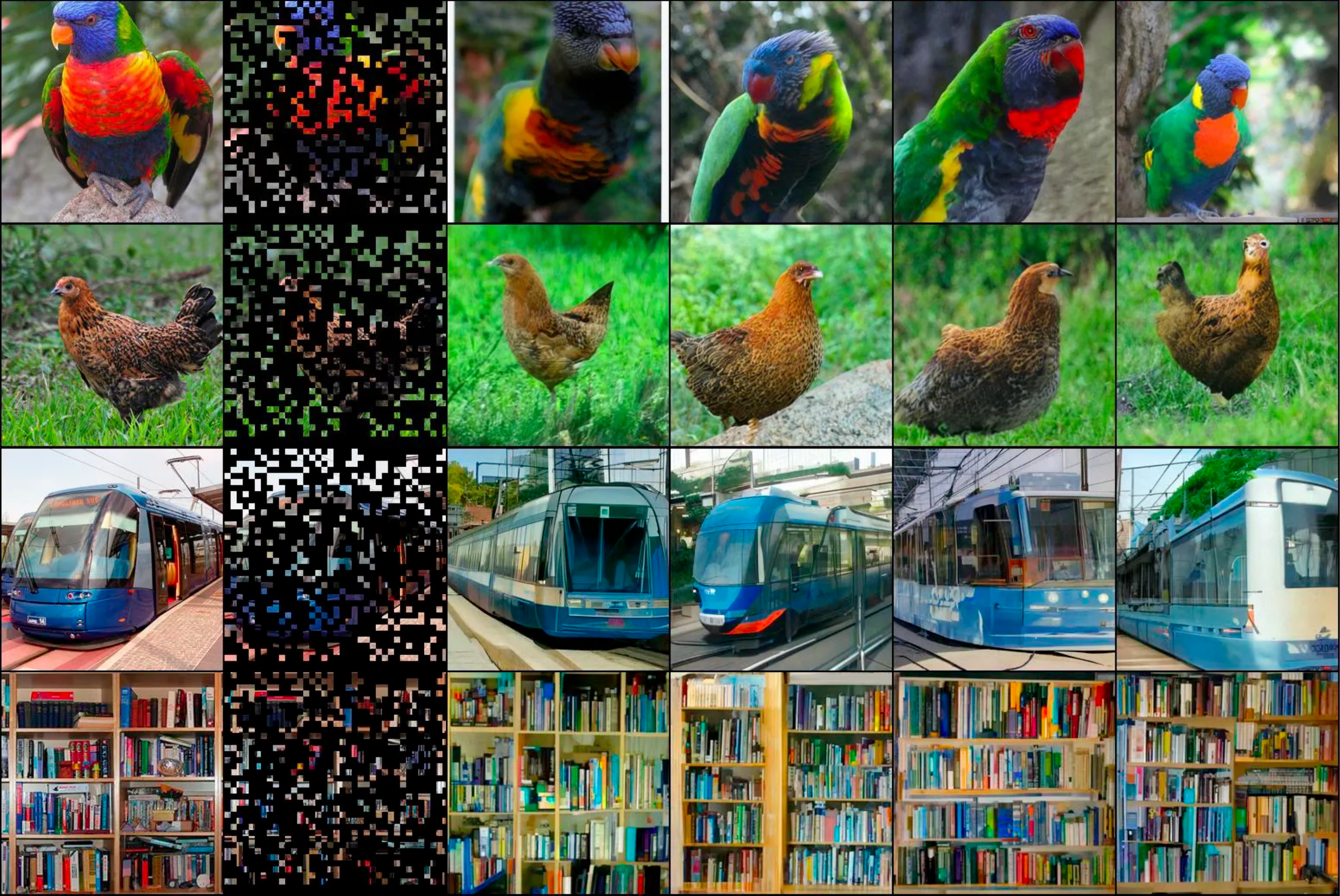}
    \caption{{\bf Visualization of MSN representations.} First column: original image. Second column: image with 70\% of patches masked, input to an MSN pre-trained ViT-L/7 encoder to compute representations. Other columns: Samples of a generative model conditioned on the MSN representations (see~Appendix~\ref{apndx:qualitative} for more details and other samples). Qualities that vary across samples represent information that the pre-trained representation is invariant to; e.g., in this case, MSN discards background, pose, and lighting information. Qualities that are common across samples represent information that the pre-trained representation is not invariant to. In this case, even with a large fraction of the patches corrupted with mask noise, MSN representations still encode semantic information about the object of interest.}
    \label{fig:msn_sampling}
\end{figure}

Empirically, we demonstrate that MSNs learn strong off-the-shelf representations that excel at low-shot prediction (cf.~Figure~\ref{fig:lowshot}).
In particular, MSN achieves good classification performance using $100\times$ fewer labels than current mask-based auto-encoders~\citep{he2021masked,xie2019unsupervised}.
In the standard 1\% ImageNet low-shot classification task, an MSN-trained ViT-B/4 (using a patch size of $4$x$4$ pixels) achieves 75.7\% top-1 accuracy, outperforming the previous 800M parameter state-of-the-art convolutional network~\citep{chen2020big} while using nearly $10\times$ fewer parameters (cf.~Figure~\ref{fig:1percent}).

Since a good representation should not need many examples to learn about a concept~\citep{goyal2019scaling}, we also consider a more challenging evaluation benchmark for label-efficient low-shot classification~\citep{sohn2020fixmatch,lucas2021barely}, using from 1 labeled image per class up to 5 images per class (cf.~Table~\ref{tb:lowshot_imagenet}).
MSN also achieves state-of-the-art performance in that regime.
For instance, with only 5 labeled images per class, we can pre-train a ViT-L/7 with MSN on ImageNet-1K to achieve 72.1\% top-1 accuracy surpassing the previous state-of-the-art method, DINO~\citep{caron2021emerging}, by 8\% top-1. 

Similar to masked auto-encoders, MSNs also exhibit good computational scaling since only the unmasked patches are processed by the ViT encoder.
For example, by randomly masking 70\% of the patches, MSN uses half the computation and memory compared to an unmasked joint-embedding baseline. In practice, we pre-train a ViT-L/7 on as few as 18 AWS {\tt p4d-24xlarge} machines. Without masking, the same job requires over 42 machines.

Finally, we also show that MSNs are competitive with prior works on other self-supervised benchmarks that use many labels for evaluation (e.g., fine-tuning, linear-evaluation, transfer learning).

\section{Prerequisites}

\paragraph{Problem Formulation}
Consider a large collection of unlabeled images, $\mathcal{D}=(\mathbf{x}_i)_{i=1}^U$, and a small dataset of annotated images, $\mathcal{S}=({\mathbf{x}_{s}}_i, y_i)_{i=1}^L$, with $L \ll U$.
Here, the images in $\mathcal{S}$ may overlap with the images in the dataset $\mathcal{D}$.
Our goal is to learn image representations by first pre-training on $\mathcal{D}$ and then adapting the representation to the supervised task using $\mathcal{S}$.

\paragraph{Siamese Networks}
The goal of siamese networks~\citep{becker1992self,bromley1993signature}, as they are used in self-supervised learning, is to learn an encoder that produces similar image embeddings for two views of an image.
Specifically, given an encoder $f_\theta(\cdot)$ and two views $\mathbf{x_i}$ and $\mathbf{x^+_i}$ of an image, the encoder independently processes each view and outputs representations $z_i$ and $z^+_i$ respectively, referred to as the anchor representation and the target representation.
The objective of siamese networks is to learn an encoder that is not sensitive to differences between views, so the representations $z_i$ and $z^+_i$ should match. %produce an anchor representation that matches the target representation.
In practice, the encoder $f_\theta(\cdot)$ is usually parameterized as a deep neural network with learnable parameters $\theta$.

The main challenge with siamese architectures is to prevent representation collapse in which the encoder produces a constant image embedding regardless of the input.
Several approaches have been investigated in the literature. Contrastive losses explicitly push away embeddings of different images~\citep{bromley1993signature,he2019moco,chen2020exploring}. Information maximization approaches try to maximize the entropy of the average prediction~\citep{caron2021emerging,assran2021semi} or spread out the embeddings uniformly on the surface of a sphere~\citep{caron2020unsupervised}.
Asymmetric approaches rely on an asymmetric architectural choice such as stop-gradient operations and a momentum encoder~\citep{chen2020exploring,grill2020bootstrap} to prevent collapse.
Other approaches try to decorrelate the vector components of the embeddings to minimize redundancy across samples~\citep{zbontar2021barlow,bardes2021vicreg}.

\paragraph{Vision Transformer}
We use a standard Vision Transformer (ViT) architecture~\citep{dosovitskiy2020image} as the encoder.
Vision Transformers first extract a sequence of non-overlapping patches of resolution $N \times N$ from an image.
Next, they apply a linear layer to extract patch tokens, and subsequently add learnable positional embeddings to them.
An extra learnable {\tt [CLS]} token is added to the sequence.
This token aims to aggregate information from the full sequence of patches~\citep{dosovitskiy2020image,caron2021emerging}.
The sequence of tokens is then fed to a stack of Transformer layers~\citep{vaswani2017attention}.
A Transformer layer is composed of a self-attention~\citep{vaswani2017attention} and a fully-connected layer with skip connections~\citep{he2016deep}.
Self-attention uses an attention mechanism~\citep{bahdanau2014neural} applied to the entire sequence of elements to update the representation.
The output representation associated to the {\tt [CLS]} token is used as the output of the encoder.

\section{Masked Siamese Networks}
\label{sec:methodology}

\begin{figure}[t]
    \centering
    \includegraphics[width=0.95\linewidth]{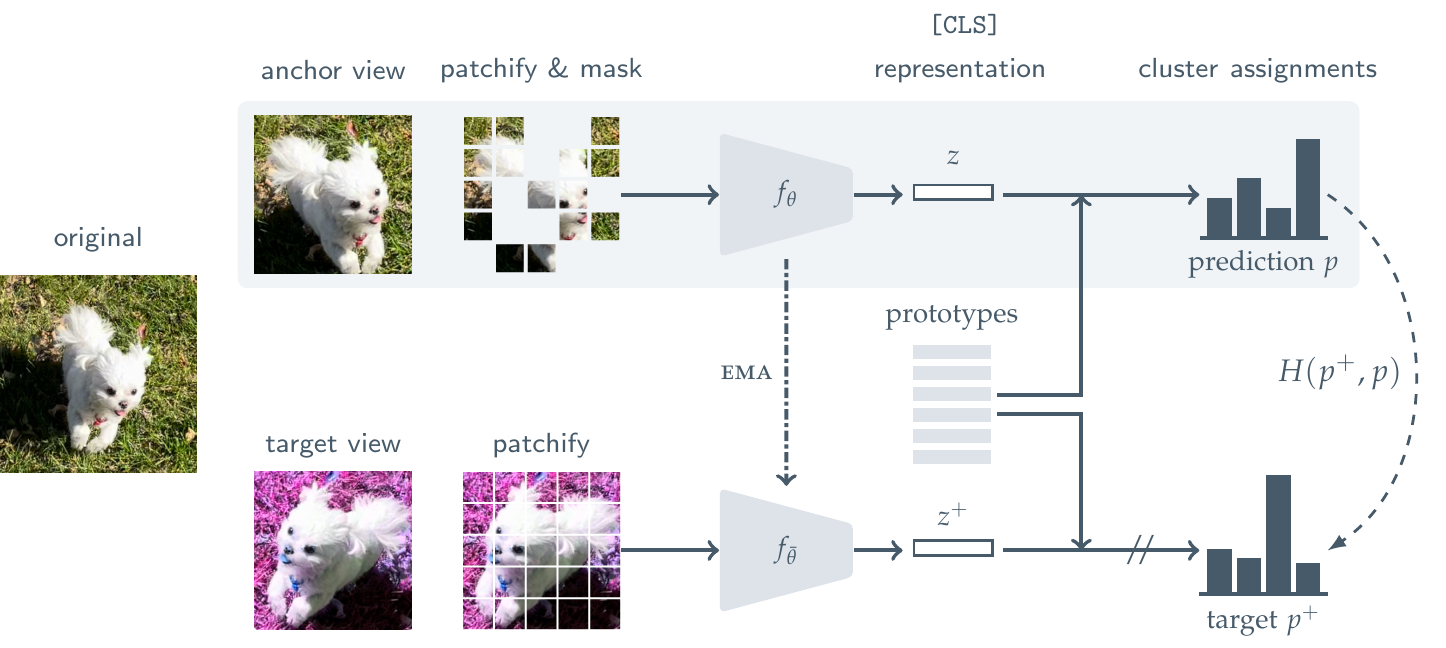}
    \caption{{\bf Masked Siamese Networks.} First use random data augmentations to generate two views of an image, referred to as the anchor view and the target view. Subsequently, a random mask is applied to the anchor view, while the target view is left unchanged. The objective is then to assign the representation of the masked anchor view to the same clusters as the representation of the unmasked target view. A standard cross-entropy loss is used as the criterion to optimize.}
    \label{fig:msn}
\end{figure}

We now describe the proposed Masked Siamese Network (MSN) training procedure, which combines invariance-based pre-training with mask denoising; see Figure~\ref{fig:msn} for a schematic.
MSNs first use random data augmentations to generate two views of an image, referred to as the anchor view and the target view.
Subsequently, a random mask is applied to the anchor view, while the target view is left unchanged.
Similar to clustering-based SSL approaches~\citep{caron2020unsupervised,caron2021emerging,assran2021semi}, learning occurs by computing a soft-distribution over a set of prototypes for both the anchor and target views. The objective is then to assign the representation of the masked anchor view to the same prototypes as the representation of the unmasked target view. We use a standard cross-entropy loss to optimize this criterion.

In contrast to previous work on masked image modelling, the mask-denoising process in MSN is discriminative, rather than generative~\citep{he2021masked,xie2021simmim,wei2021masked,bao2021beit,zhou2021ibot}.
MSN architectures do not directly predict pixel values (or tokens) for the masked patches.
Instead, the loss is applied directly to the output corresponding to the {\tt [CLS]} token of the encoder.

%\subsection{Detailed Methodology}
\paragraph{Input Views}
In each iteration of pre-training, we sample a mini-batch of $B \geq 1$ images.
For an index $i \in [B]$, let ${\bf x}_i$ denote the $i^{\text{th}}$ image in the mini-batch.
For each image ${\bf x}_i$, we first apply a random set of data augmentations to generate a target view, denoted ${\bf x}^+_i$, and $M \geq 1$ anchor views, denoted ${\bf x}_{i,1}, {\bf x}_{i,2}, \ldots, {\bf x}_{i,M}$.

\paragraph{Patchify and Mask}
Next, we ``patchify'' each view by converting it into a sequence of non-overlapping $N \times N$ patches. After patchifying the anchor view ${\bf x}_{i,m}$, we also apply the additional step of masking by randomly dropping some of the patches. We denote by ${\bf \hat{x}}_{i,m}$ the sequence of masked anchor patches, and by ${\bf \hat{x}}^+_i$ the sequence of unmasked target patches. Because of masking, the anchor sequence ${\bf\hat{x}}_{i,m}$ can have a different length than the patchified target sequence ${\bf\hat{x}}^+_i$, even if both image views originally have the same resolution.
\begin{figure}[t]
    \centering
    \begin{subfigure}{0.225\linewidth}
        \centering
        \includegraphics[width=0.9\linewidth]{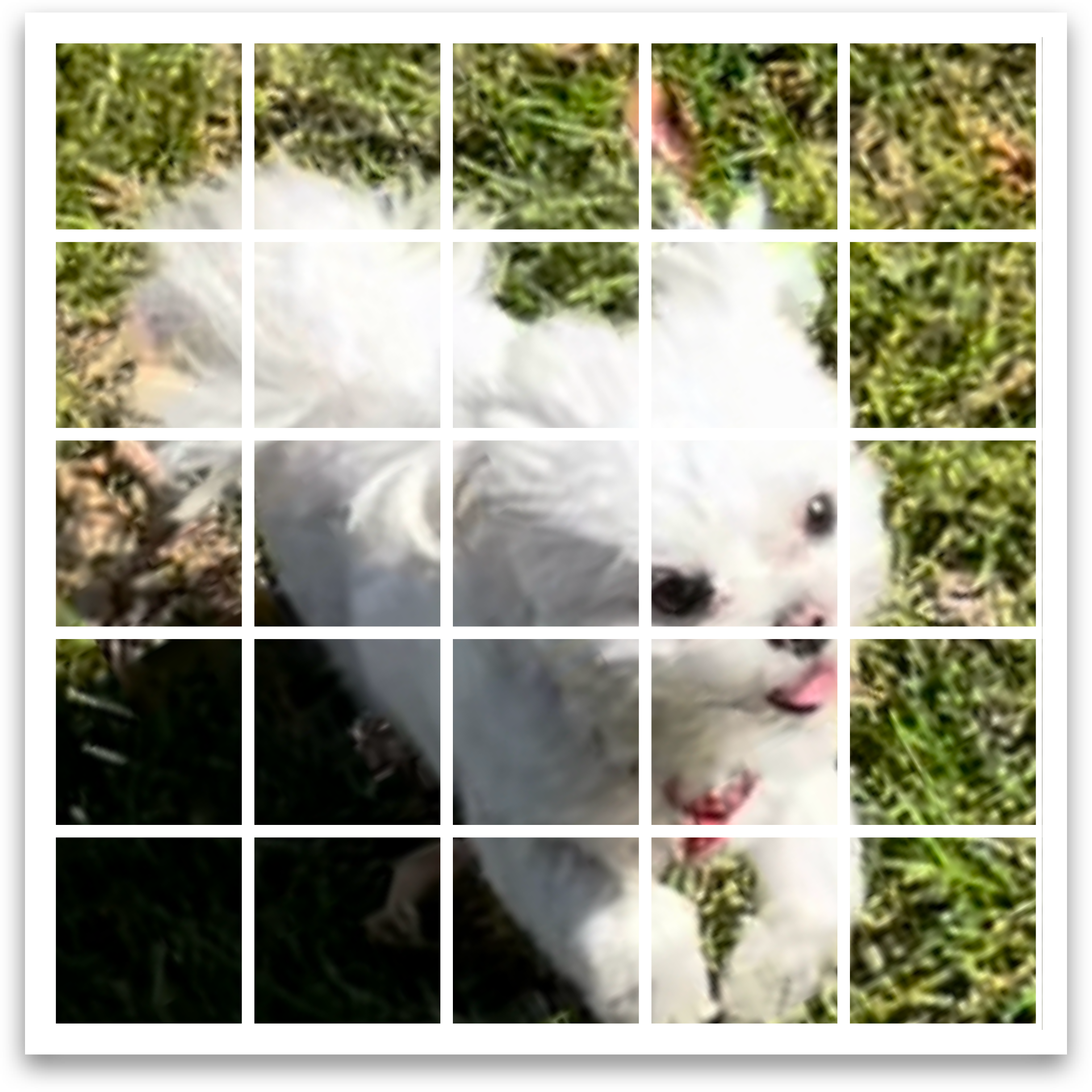}
        \caption{\scriptsize No Mask}
    \end{subfigure}
    \begin{subfigure}{0.225\linewidth}
        \centering
        \includegraphics[width=0.9\linewidth]{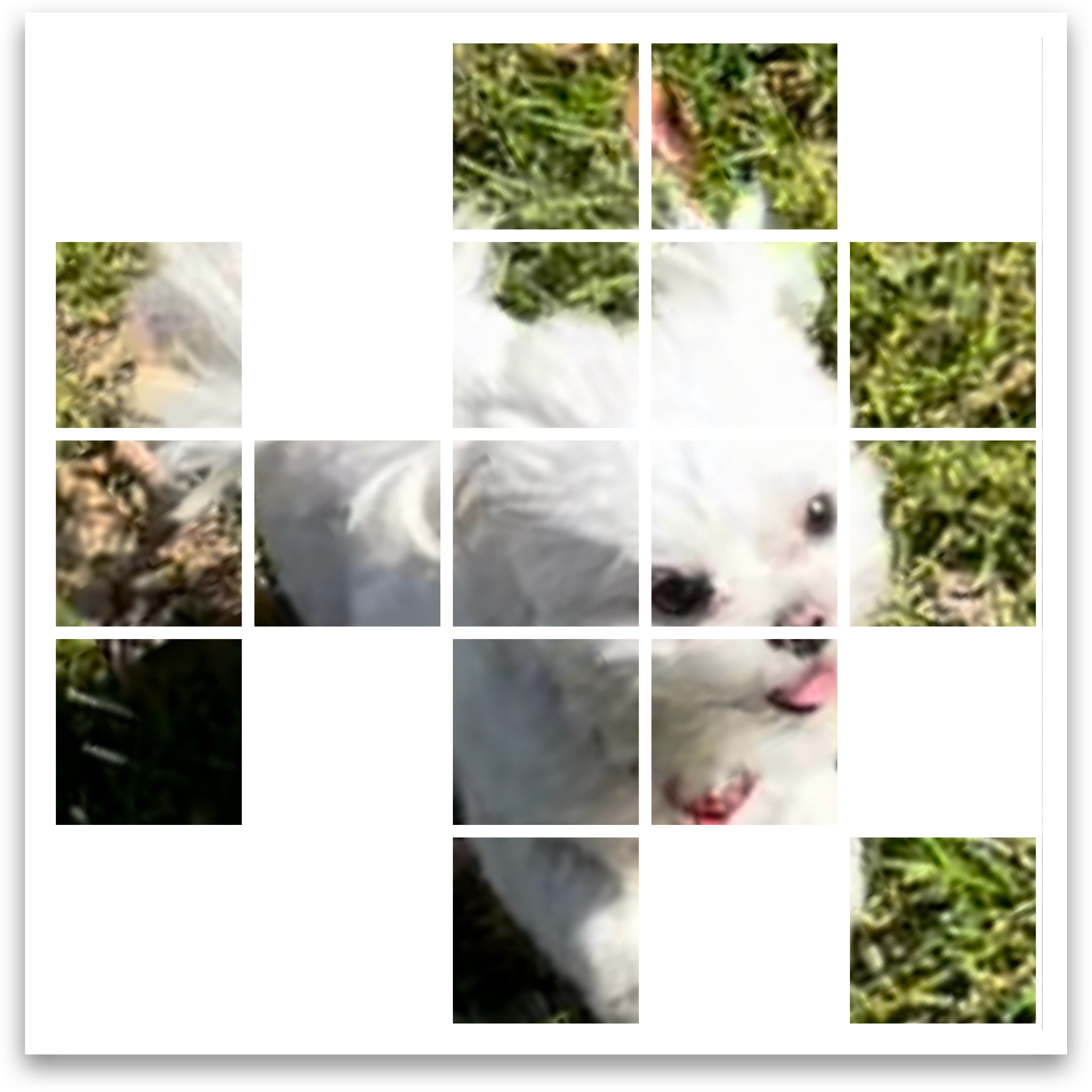}
        \caption{\scriptsize Random Mask}
    \end{subfigure}
    \begin{subfigure}{0.225\linewidth}
        \centering
        \includegraphics[width=0.9\linewidth]{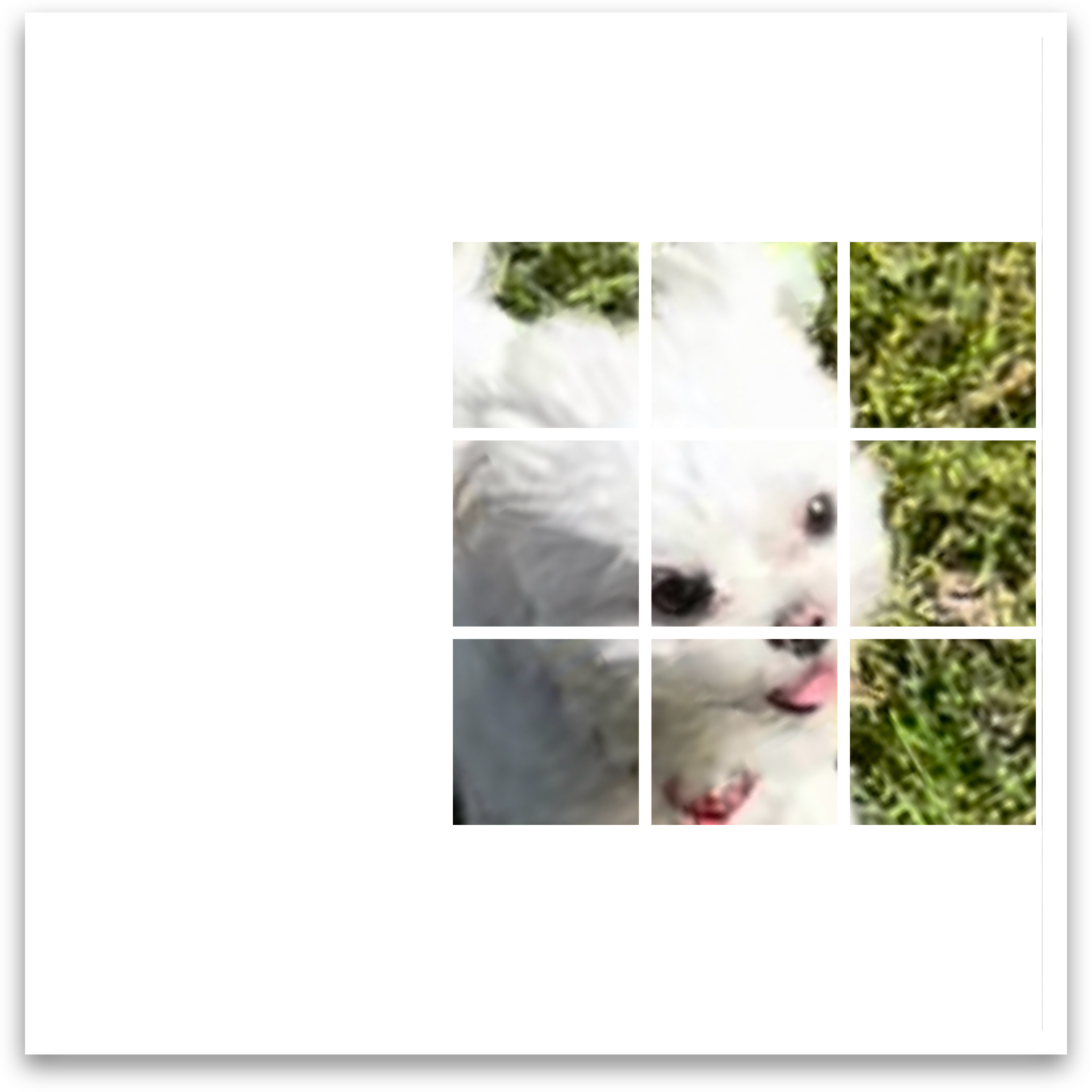}
        \caption{\scriptsize Focal Mask}
    \end{subfigure}
    \caption{{\bf Masking strategies.} When applying a Random Mask, we randomly drop patches across a global view of the image. When applying a Focal Mask, we randomly select a local continuous block of an image, and mask everything around it. We typically leverage both Random and Focal Masking strategies when pre-training with MSNs.}
    \label{fig:masking_strategies}
\end{figure}

We investigate two strategies for masking the anchor views, Random Masking and Focal Masking, which are depicted in Figure~\ref{fig:masking_strategies}.
When applying Random Masking, we randomly drop potentially non-contiguous patches across the sequence. 
Conversely, when applying Focal Masking, we randomly select a local continuous block of the anchor view and drop all the patches around it.

\paragraph{Encoder}
Given a parameterized anchor encoder, denoted $f_\theta(\cdot)$, let $z_{i,m} \in \R^{d}$ denote the representation computed from the patchified (and masked) anchor view ${\bf \hat{x}}_{i,m}$.
Similarly, given a parameterized target encoder $f_{\bar{\theta}}(\cdot)$, with a potentially different set of parameters $\bar{\theta}$, let $z^+_i \in \R^{d}$ denote the representation computed from the patchified target view ${\bf \hat{x}}^+_i$.
In MSNs, the parameters $\bar{\theta}$ of the target encoder are updated via an exponential moving average of the anchor encoder parameters~\citep{grill2020bootstrap}.
Both encoders correspond to the trunk of a ViT~\citep{dosovitskiy2020image}.
We take the output of the network to be the representation corresponding to the {\tt [CLS]} token.%\footnote{Note that when computing anchor representations with a ViT, the patches utilize the positional embeddings corresponding to their original location in the unmasked image.}

\paragraph{Similarity Metric and Predictions}
Let ${\bf q} \in \R^{K \times d}$ denote $K > 1$ learnable prototypes, each of dimension $d$.
To train the encoder, we compute a distribution based on the similarity between these prototypes and each anchor and target view pair, and we penalize the encoder for differences between these distributions.
% Our numerical objective in this work is to compute a soft-distribution over these prototypes for each anchor-target view pair and encourage them to match.
More precisely, for an anchor representation $z_{i,m}$, we compute a ``prediction'' $p_{i,m} \in \Delta_K$ in the $K$-dimensional simplex by measuring the cosine similarity to the prototypes matrix ${\bf q}$.
For $L_2$-normalized representations and prototypes, the predictions $p_{i,m}$ can be concisely written as
\[
    p_{i,m} \defeq \text{softmax}\left( \frac{z_{i,m} \cdot {\bf q}}{\tau} \right),
\]
where $\tau \in (0, 1)$ is a temperature.
Similarly, for each target representation $z^+_i$, we generate a prediction $p^+_i \in \Delta_K$ by measuring the cosine similarity to the same prototypes matrix ${\bf q}$.
When computing the target predictions, we also use a temperature parameter $\tau^+  \in (0, 1)$.
Note, we always choose $\tau^+ < \tau$ to encourage sharper target predictions, which implicitly guides the model to produce confident low entropy anchor predictions.
As shown in Appendix~\ref{apndx:theory}, target sharpening coupled with with other regularization like mean-entropy maximization (see below) is provably sufficient to eliminate collapsing solutions in the MSN framework.
Empirically, we have observed that training without sharpening can result in collapsing solutions.

\paragraph{Training Objective} 
As previously mentioned, to train the encoder, we penalize when the anchor prediction $p_{i,m}$ is different from the target prediction $p_i^+$. 
We enforce this criterion using a standard cross-entropy loss $H(p^+_i, p_{i,m})$.

We also incorporate the mean entropy maximization ({\sc me-max}) regularizer, also used in~\citep{assran2021semi,joulin2012convex}, to encourage the model to utilize the full set of prototypes.
Denote the average prediction across all the anchor views by
\[ 
    \overline{p} \defeq \frac{1}{MB}\sum^B_{i=1}\sum^M_{m=1} p_{i,m}.
\]
The {\sc me-max} regularizer simply seeks to maximize the entropy of $\overline{p}$, denoted $H(\overline{p})$, or equivalently, minimize the negative entropy of $\overline{p}$.
Thus, the overall objective to be minimized when training the encoder parameters $\theta$ and prototypes $\mathbf{q}$ is
\begin{equation}
    \label{eq:objective}
    \frac{1}{MB} \sum^B_{i=1}\sum^M_{m=1} H(p^+_i, p_{i,m}) - \lambda H(\overline{p}),
\end{equation}
where $\lambda > 0$ controls the weight of the  {\sc me-max} regularization.
Note that when training, we only compute gradients with respect to the anchor predictions $p_{i,m}$, not the target predictions $p^+_i$.

\section{Related Work}
\label{sec:related-work}

Unsupervised pre-training for vision has seen rapid progress  with the development of view-invariant representation learning and joint embedding architectures%, which can be trained using either contrastive or non-contrastive loss functions
~\citep{wu2018unsupervised,he2019moco,chen2020exploring,grill2020bootstrap,caron2021emerging,bardes2021vicreg}.
Most similar to our approach is DINO~\citep{caron2021emerging} which leverages a Siamese Network with a cross-entropy loss and a momentum encoder.
DINO also uses multi-crop training, which is a form of focal masking, but it requires an unmasked anchor view during training.
MSN can be seen as a generalization of DINO, leveraging both random and focal masking without requiring any unmasked anchor views. Since the cross-entropy loss in equation~\eqref{eq:objective} is only differentiated with respect to the anchor predictions, not the target, MSN only backpropagates through the anchor network and only needs to store the activation associated with the masked view. MSN therefore reduces the computational and memory requirements.
%Hence, the memory overhead for processing an unmasked target view is quite small.
MSN also differs from DINO in its mechanism for preventing representation collapse (entropy maximization as opposed to centering and sharpening).
Our empirical results show that MSN compares favourably to DINO across various degrees of supervision for the downstream task.

A prominent line of work in SSL is to remove a portion of the input and learn to reconstruct the removed content~\citep{devlin2018bert}.
For example, in the field of image recognition, some works have proposed to predict augmented image channels~\citep{zhang2017split}, which can be regarded as a form of image colorization~\citep{zhang2016colorful,larsson2016learning,larsson2017colorization}.
Other approaches propose to remove and learn to regress entire image regions: the seminal Context Encoders of~\citet{pathak2016context} train a network to generate missing image patches based on their surroundings.
Recent works revisit this idea and investigate the pre-training of ViTs with masked auto-encoders~\citep{chen2020generative,he2021masked,xie2021simmim,wei2021masked,bao2021beit}.
These approaches corrupt images with mask-noise and predict missing input values at the pixel level~\citep{dosovitskiy2020image,he2021masked,xie2019unsupervised} or using a tokenizer~\citep{bao2021beit,wei2021masked}.
Our approach does not predict the missing value at the input level, but instead performs the denoising step implicitly by ensuring that the global representation of the noisy input matches that of the uncorrupted input.

Some recent approaches have started to explore the combination of joint-embedding architectures and denoising pre-training tasks~\citep{el2021large,baevski2022data2vec,zhou2021ibot}.
Those approaches mask an image by replacing the masked patches with a learnable mask token, and output a single vector for each masked patch.
The objective is then to directly match each computed patch vector to the equivalent patch token extracted from a target encoder.
In addition to the patch-level loss, iBOT~\citep{zhou2021ibot} and SplitMask~\citep{el2021large}  apply a joint-embedding loss to an output representing the global sequence (either the {\tt [CLS]} token or a global average pool of the patch vectors). SplitMask shows that by using a patch-level loss, you can reduce the amount of unlabeled pre-training data. In contrast, we focus on reducing the amount of labeled data available for the downstream prediction task.
Data2Vec~\citep{baevski2022data2vec} demonstrates that this approach is suitable for multiple modalities such as vision, speech and text.
Different from these approaches, we only match the view representations globally and do not consider a patch level loss.
Consequently, we can completely ignore the masked patches, significantly reducing the computational and memory requirements.
For example, when training our largest model, a ViT-L/7, we mask over 70\% of the input patches, and reduce memory and computational overhead by  half.

\section{Results}
\label{sec:results}

We evaluate MSN representations learned on the ImageNet-1K dataset~\citep{russakovsky2015imagenet}.
%Since one of SSL goals is to learn representation that can be applied to tasks with few labels, 
We first consider low-shot evaluation on ImageNet-1K using as few as 1--5 images per class.
We also compare with the state-of-the-art in settings where more supervision is available and investigate transfer-learning performance. Finally, we conduct ablation experiments with MSN. By default, we pre-train with a batch-size of 1024 images, generating several anchor views from each image: 1 view with a random mask, and 10 views with focal masks. We find that the optimal masking ratio is model-dependent, with larger models benefiting from more aggressive patch dropping. We describe MSN implementation details in Appendix~\ref{apndx:implementation}.
\begin{table}[t]
    \centering
    \caption{\textbf{Extreme low-shot.} We evaluate the label-efficiency of self-supervised models pretrained on the ImageNet-1K dataset. For evaluation, we use an extremely small number of the ImageNet-1K labels and report the mean top-1 accuracy and standard deviation across 3 random splits of the data.}
    \label{tb:lowshot}
    \begin{tabular}{l l l c c c}
        % & & & \multicolumn{3}{c}{\bf\small Top 1}\\[1.5mm]
        & & & \multicolumn{3}{c}{\small Images per Class}\\
        \bf\small Method & \bf\small Architecture & \bf\small Epochs & 1 & 2 & 5 \\\toprule
        \multirow{2}{*}{iBOT~\citep{zhou2021ibot}} & ViT-S/16 & 800 & 40.4 $\pm$ 0.5 & 50.8 $\pm$ 0.8 & 59.9 $\pm$ 0.2 \\
        & ViT-B/16 & 400 & 46.1 $\pm$ 0.3 & 56.2 $\pm$ 0.7 & 64.7 $\pm$ 0.3 \\[3mm]
        \multirow{4}{*}{DINO~\citep{caron2021emerging}} & ViT-S/16 & 800 & 38.9 $\pm$ 0.4 & 48.9 $\pm$ 0.3 & 58.5 $\pm$ 0.1 \\
        & ViT-B/16 & 400 & 41.8 $\pm$ 0.3 & 51.9 $\pm$ 0.6 & 61.4 $\pm$ 0.2 \\[1mm]
        & ViT-S/8 & 800 & 45.5 $\pm$ 0.4 & 56.0 $\pm$ 0.7 & 64.7 $\pm$ 0.4 \\
        & ViT-B/8 & 300 & 45.8 $\pm$ 0.5 & 55.9 $\pm$ 0.6 & 64.6 $\pm$ 0.2 \\[3mm]
        \multirow{3}{*}{MAE~\citep{he2021masked}} & ViT-B/16 & 1600 & 8.2 $\pm$ 0.3 & 25.0 $\pm$ 0.3 & 40.5 $\pm$ 0.2 \\
        & ViT-L/16 & 1600 & 12.3 $\pm$ 0.2 & 19.3 $\pm$ 1.8 & 42.3 $\pm$ 0.3 \\[1mm]
        & ViT-H/14 & 1600 & 11.6 $\pm$ 0.4 & 18.6 $\pm$ 0.2 & 32.8 $\pm$ 0.2 \\ \midrule
        \multirow{5}{*}{MSN (Ours)} & ViT-S/16 & 800 & 47.1 $\pm$ 0.1 & 55.8 $\pm$ 0.6 & 62.8 $\pm$ 0.3 \\
        & ViT-B/16 & 600 & 49.8 $\pm$ 0.2 & 58.9 $\pm$ 0.4 & 65.5 $\pm$ 0.3 \\[1mm]
        % & ViT-L/16 & 600 & 48.3 $\pm$ 0.2 & 57.8 $\pm$ 0.3 & 66.0 $\pm$ 0.2 \\[1mm]
        & ViT-B/8 & 600 & 55.1 $\pm$ 0.1 & 64.9 $\pm$ 0.7 & 71.6 $\pm$ 0.3 \\
        & ViT-B/4 & 300 & 54.3 $\pm$ 0.4 & 64.6 $\pm$ 0.7 & \cellcolor{fbApp}\bf 72.4 $\pm$ 0.3 \\[1mm]
        & ViT-L/7 & 200 & \cellcolor{fbApp}\bf 57.1 $\pm$ 0.6 & \cellcolor{fbApp}\bf 66.4 $\pm$ 0.6 & \cellcolor{fbApp}\bf 72.1 $\pm$ 0.2 \\
        \bottomrule
    \end{tabular}
\end{table}
\begin{table}[t]
    \caption{Low-shot evaluation on ImageNet-1K using 1\% of the labels (approximately 13 images per class). $^\dagger$Indicates evaluations we computed  using publicly available models.}
    \label{tb:lowshot_imagenet}
    \centering
    \begin{tabular}{l l l c}
        \bf Method & \bf Architecture & \bf Params. & \bf Top 1 \\\toprule
        \multicolumn{4}{c}{\scriptsize\bf Comparing similar architectures}\\[1mm]
        Barlow-Tw.~\citep{zbontar2021barlow} & RN50 & 24M & 55.0 \\
        SimCLRv2~\citep{chen2020big} & RN50 & 24M & 57.9 \\
        PAWS~\citep{assran2021semi} & RN50 & 24M & 66.5 \\[1mm]
        DINO~\citep{caron2021emerging} & ViT-S/16 & 22M & 64.5 \\
        iBOT~\citep{zhou2021ibot} & ViT-S/16 & 22M & 65.9 \\\midrule
        MSN & ViT-S/16 & 22M & \cellcolor{fbApp}\bf 67.2 \\
        \midrule \midrule
        \multicolumn{4}{c}{\scriptsize\bf Comparing larger architectures}\\[1mm]
        BYOL~\citep{grill2020bootstrap} &  RN200 ($2\times$) & 250M & 71.2 \\
        SimCLRv2~\citep{chen2020big} & RN151+SK ($3\times$) & 795M & 74.9 \\[1mm]
        %MAE$^\dagger$ & ViT-H/14 & 632M & 46.7 \\
        iBOT~\citep{zhou2021ibot}$^\dagger$ & ViT-B/16 & 86M & 69.7 \\
        DINO~\citep{caron2021emerging}$^\dagger$ & ViT-B/8 & 86M & 70.0 \\\midrule
        \multirow{2}{*}{MSN} & ViT-L/7 & 304M & 75.1 \\
        & ViT-B/4 & 86M & \cellcolor{fbApp}\bf 75.7 \\
        \bottomrule
    \end{tabular}
\end{table}

\subsection{Label-Efficient Learning}
The premise of SSL is to learn representations on unlabeled data that can be effectively applied to prediction tasks with few labels~\citep{chen2020big}. In this section we explore the performance of self-supervised approaches when very few labeled examples are available.

\paragraph{Extreme Low-Shot}
% \paragraph{1--5 labeled images per class}
% \im{Do we want to call this extreme low-shot to distinguish tables 1 and 2? Also helps us highlight this setting by a single name in the contributions.}
We first evaluate the classification performance of unsupervised models that have been pre-trained on ImageNet-1K, by using 1, 2, and 5 labeled images per class for supervised evaluation. We compare MSN to the joint-embedding approach, DINO~\citep{caron2021emerging}, the auto-encoding approach, MAE~\citep{he2021masked}, and the hybrid approach, iBOT~\citep{zhou2021ibot}, which combines a joint-embedding architecture with a token-based patch-level loss.
We download the official released models of each related approach for evaluation.

To adapt the joint-embeddings models to the supervised task, we freeze the weights of the pre-trained model and train a linear classifier on top using 1, 2 or 5 labeled samples (see Appendix~\ref{apndx:implementation}).
For MAE, we rely on partial fine-tuning~\citep{he2021masked}, except for the 1 image per class setting, and all results with the ViT-H/14 architecture, which use a linear classifier.
Partial fine-tuning corresponds to fine-tuning the last block of the pre-trained model along with a linear head. MAE benefits from partial fine-tuning, but for sufficiently large models, such as the ViT-H/14, this leads to significant overfitting in the low-shot regime. We compare both protocols in more detail in Appendix~\ref{apndx:additional_ablations}. 

Table~\ref{tb:lowshot} reports the extreme low-shot evaluation results.
MSN outperforms the other representation learning approaches across all levels of supervision.
Moreover, the improvement offered by MSN increases as the amount of available labeled data is decreased.
The performance of MSN also benefits from increased model size --- settings with less labeled data appear to benefit more from increased model depth and smaller patch sizes.

We also observe that joint-embedding approaches appear to be more robust to the limited availability of downstream supervision than reconstruction-based auto-encoding approaches.
To explain this observation, we refer to the Masked Auto-Encoders paper~\citep{he2021masked} which conjectures that using a pixel reconstruction loss results in encoder representations of a lower semantic level than other methods.
Conversely, the inductive bias introduced by invariance-based pre-training appears to be helpful in the low-shot regime.

\paragraph{1\% ImageNet-1K}
%In the low-shot learning evaluation benchmark, an unsupervised pre-trained model is evaluated with a small amount of labeled images; i.e., 1\% of ImageNet-1K, which is equivalent to roughly 13 images per class.
%For reproducibilty, we use the same 1\% data split used in previous work~\citep{chen2020simple}.
Table~\ref{tb:lowshot_imagenet} reports a comparison on the 1\% ImageNet-1K task, which is a standard benchmark for low-shot evaluation of self-supervised models~\citep{chen2020simple}.
For reference, the best reported result in the literature on 1\% labeled data is 76.6\%, achieved with a multi-stage semi-supervised pipeline, i.e., self-distilling from a fine-tuned ResNet-152 with 3$\times$ wider channels and selective kernels~\citep{chen2020big}.
Here we focus on comparing to other models trained in a self-supervised setting.
Our best MSN model using a ViT-B/4 achieves 75.7\% top 1 accuracy, surpassing the previous 800M parameter state-of-the-art convolutional network~\citep{chen2020big} while using significantly fewer parameters and no fine-tuning.
When focusing the comparison on similar architectures (models with similar FLOP counts), MSN also consistently improves upon previous approaches. 
\begin{table}[t]
    \caption{Linear evaluation on ImageNet-1K using 100\% of the labels.}
    \label{tb:vit-s_imagenet}
    \centering
    \begin{tabular}{l l l l c}
        \bf Method & \bf Architecture & \bf Params. & \bf Epochs & \bf Top 1 \\\toprule
        \multicolumn{5}{c}{\scriptsize\bf Comparing similar architectures}\\[2mm]
        SimCLRv2~\citep{chen2020big} & RN50 & 24M & 800 & 71.7 \\
        %Barlow-Tw. & RN50 & 24M & 1000 & 73.2 \\
        BYOL~\citep{grill2020bootstrap} & RN50 & 24M & 1000 & 74.4 \\
        DINO~\citep{caron2021emerging} & ViT-S/16 & 22M & 800 & 77.0 \\
        iBOT~\citep{zhou2021ibot} & ViT-S/16 & 22M & 800 & \bf 77.9 \\
        MSN & ViT-S/16 & 22M & 600 & \cellcolor{fbApp} 76.9 \\
        \midrule \midrule
        \multicolumn{5}{c}{\scriptsize\bf Comparing larger architectures}\\[2mm]
        % \multirow{3}{*}{BYOL} & RN50 ($2\times$) & 94M & 800 & 77.4 \\
        MAE~\citep{he2021masked} & ViT-H/14 & 632M & 1600 & 76.6 \\
        BYOL~\citep{grill2020bootstrap} &  RN200 ($2\times$) & 250M & 800 & 79.6 \\
        SimCLRv2~\citep{chen2020big} & RN151+SK ($3\times$) & 795M & 800 & 79.8 \\
        \multirow{1}{*}{iBOT}~\citep{zhou2021ibot} & ViT-B/16 & 86M & 400 & 79.4 \\
        %\multirow{3}{*}{DINO} & ViT-B/16 & 86M & 400 & 78.2 \\[1mm]
        %& ViT-S/8 & 22M & 800 & 79.7 \\
        DINO~\citep{caron2021emerging} & ViT-B/8 & 86M & 300 & 80.1 \\
        %\multirow{4}{*}{MoCov3} & ViT-B/16 & 86M & 600 & 76.7 \\
        %& ViT-L/16 & 304M & 300 & 77.6 \\
        %& ViT-H/16 & 632M & 300 & 78.1 \\
        %& ViT-BN-H/16 & 632M & 300 & 79.1 \\[1mm]
        MoCov3~\citep{chen2021empirical}& ViT-BN-L/7 & 304M & 300 & \bf 81.0 \\
        %\multirow{3}{*}{MAE} & ViT-B/16 & 86M & 1600 & 68.0 \\
        %& ViT-L/16 & 304M & 1600 & 75.8  \\[1mm]
        %& ViT-H/14 & 632M & 1600 & 76.6 \\[4mm]
        % \multirow{6}{*}{MSN} & ViT-B/16 & 86M & 600 & -- \\
        % & ViT-L/16 & 304M & 300 & (78.2) \\[1mm]
        % & ViT-S/8 & 22M & 300 & --\\
        % & ViT-B/8 & 86M & 800 & (78.8) \\[1mm]
        %\multirow{2}{*}{MSN} & ViT-B/7 & 86M & 600 & \cellcolor{fbApp} (80.4) \\
        MSN & ViT-L/7 & 304M & 200 & \cellcolor{fbApp} \bf 80.7 \\
        \bottomrule
    \end{tabular}
\end{table}
\begin{table}[t]
    \centering
    \caption{End-to-end fine-tuning of a ViT-B/16 encoder on ImageNet-1K using 100\% of the labels. MSN obtains competitive performance with both joint-embedding approaches and auto-encoding approaches.}
    \label{tb:finetuning_imagenet}
    \begin{tabular}{l c c}
        {\bf \small Initialization} & {\bf \small Pretrain Epochs} & {\bf \small Top 1}\\\toprule
        DINO~\citep{caron2021emerging} & 800 & 83.6\\
        BEiT~\citep{bao2021beit}\ & 800 & 83.2\\
        iBOT~\citep{he2021masked} & 800 &83.8\\
        MAE~\citep{he2021masked} & 1600 &83.6\\
        SimMIM~\citep{xie2021simmim} & - & 83.8\\
        MaskFeat~\citep{wei2021masked} & - & 84.0\\
        Data2Vec~\citep{baevski2022data2vec} & 800 &{\bf 84.2}\\
        MSN & 600 & \cellcolor{fbApp} 83.4\\ \bottomrule
    \end{tabular}
\end{table}

\subsection{Linear Evaluation and Fine-tuning}
%\mr{Why not move the ``Low-shot Learning'' paragraph to the previous subsection, which is already about low-shot learning, and then call this subsection something like ``Linear Evaluation and Fine-Tuning''? It seems strange to call this ``Comparison with state-of-the-art''. Wasn't the previous subsection already comparing with state-of-the-art methods? (Not all, but some...)}
In this section we compare with the state-of-the-art on standard evaluation benchmarks where more supervised samples are available to adapt the representation.
We use the full ImageNet-1K training images with 1.28M labels.

\paragraph{Linear Evaluation}
% In the linear evaluation benchmark, an unsupervised pre-trained model is evaluated by freezing its weights and training a linear classifier using the entire ImageNet-1K training set.
We evaluate self-supervised pretrained models by freezing their weights and training a linear classifier.
%using the entire 1.28M labeled images from the ImageNet-1K training set.
% In contrast to the low-shot evaluation benchmark, here models have access to all 1.28M labeled images.
Table~\ref{tb:vit-s_imagenet} reports the linear evaluation results on ImageNet-1K.
We observe that MSN performs competitively with the state-of-the-art. The best MSN model achieves 80.7\% top-1 accuracy. %\mr{Shouldn't the number of parameters in ViT-BN-L/7 be more than ViT-L/7 if the BN-L/7 uses batchnorm?}\nb{it is the same as it just replacs layernom with batch norm, it doesn't add extra layer}%, which is similar to the 81.0\% state-of-the-art obtained by MoCov3 on this benchmark.

\paragraph{Fine-Tuning}
In this evaluation setting, we finetune all the weights of the self-supervised model using all the labels from the ImageNet-1K training set.
% In the fine-tuning benchmark, an unsupervised pre-trained model is evaluated by fine-tuning its weights, along with a linear head, end-to-end, using 100\% of the ImageNet-1K labels.
%In contrast to the linear evaluation benchmark, the weights of the pre-trained model are not frozen, but instead used as an initialization for supervised learning. In these experiments, 
We focus on the ViT-B/16 architecture. We adopt the same fine-tuning protocol as~\citep{bao2021beit}, and provide the details in Appendix~\ref{apndx:implementation}.
Table~\ref{tb:finetuning_imagenet} reports the comparison with fine-tuning evaluation using 100\% labels on ImageNet-1K.
MSN is competitive with joint-embedding approaches, such as DINO, and generative auto-encoding approaches, such as MAE.

\subsection{Transfer Learning}
We also report transfer learning experiments on the CIFAR10, CIFAR100 and iNaturalist datasets in Tables~\ref{tb:transfer_ft} and~\ref{tb:transfer_lin} when using a self-supervised ViT-B/16 pre-trained on ImageNet-1K. Across all tasks and various levels of supervision MSN either outperforms or achieves similar results to DINO pre-training. Recall that MSN pre-training is also less computationally expensive than DINO pre-training due to the anchor masking.
\begin{table}[h]
    \centering
    \caption{{\bf Fine-Tuning Transfer Learning} with a ViT-Base/16 pre-trained on ImageNet-1K. Across all tasks, MSN either outperforms or achieves similar results to DINO pre-training. The MSN model is trained with a masking ratio of 0.3; i.e., dropping 30\% of patches, and thus reduces the computational cost of pre-training relative to DINO.}
    \label{tb:transfer_ft}
    \begin{tabular}{l c c c c}
    & \multicolumn{4}{c}{\bf Top 1}\\[1mm]
    \bf Method & CIFAR10 & CIFAR100 & iNat18 & iNat19 \\\toprule
    DINO & 99.0 & 90.5 & 72.0 & 78.2\\
    MSN & 99.0 & 90.5 & 72.1 & 78.1\\
    \bottomrule
    \end{tabular}
\end{table}
\begin{table}[h]
    \centering
    \caption{{\bf Linear Eval.~Transfer Learning} with a ViT-Base/16 pre-trained on ImageNet-1K. Across both tasks and various levels of supervision, MSN either outperforms or achieves similar results to DINO pre-training. The MSN model is trained with a masking ratio of 0.3; i.e., dropping 30\% of patches, and thus reduces the computational cost of pre-training relative to DINO.}
    \label{tb:transfer_lin}
    \begin{tabular}{l c c c}
    & \multicolumn{3}{c}{\bf Top 1}\\[1mm]
    & \multicolumn{2}{c}{CIFAR10} & CIFAR100 \\
    \bf Method  & {\scriptsize 4000 labels}  & {\scriptsize 50000 labels} \\\toprule
    DINO & 93.2 & 95.3 & 82.9 \\
    MSN & 93.8 & 95.7 & 82.8 \\
    \bottomrule
    \end{tabular}
\end{table}

\subsection{Ablations}
We now conduct a series of experiments to gain insights into the important design decisions used in MSN such as the masking strategy and the data augmentation strategy.
We measure the accuracy of the models by training a logistic regression classifier on the frozen trunk using 1\% of ImageNet-1K labels ($\sim$13 imgs/class).

\paragraph{Combining Random and Focal Masking}
In MSN we apply both random and focal masking to the anchor views.
Focal masking corresponds to selecting a small crop from the anchor view.
Random masking corresponds to randomly dropping potentially non-contiguous patches from the anchor view.
\begin{table}[h]
    \centering
    \caption{{\bf Masking strategy.} Impact of masking strategy on low-shot accuracy (1\% of ImageNet-1K labels) of a ViT-B/16. We only generate one anchor view of each image, except in the last row, where we generate two views, one with a Random Mask and one with a Focal Mask. A random masking ratio of $0.5$ is used. Applying  a  random  mask  to  the  anchor  view  is  better than applying no mask. By combining both random and focal masking strategies, we obtain the strongest performance.}
    \label{tb:mask_strategies}
    \begin{tabular}{lc}
      {\bf\small Anchor View} & {\bf\small Top 1} \\\toprule 
      No Mask &  49.3 \\ \midrule
      Focal Mask & 39.3  \\
      Random Mask & 52.3 \\
      Random Mask + Focal Mask & \cellcolor{fbApp}\bf 59.8 \\
      \bottomrule
    \end{tabular}
\end{table}

Table~\ref{tb:mask_strategies} reports the effect on low-shot evaluation when using a) No Masking, b) Focal Masking, c) Random Masking, or d)  Random and Focal Masking.
Applying a random mask to the anchor view is always better than applying no mask.
By contrast, applying only a focal mask degrades the performance, which highlights the importance of maintaining a global view during pre-training.
By combining both random and focal masking strategies,%, i.e., using both a masked global view and a cropped small view of the anchor, 
we obtain the strongest performance.

\paragraph{Random Masking Ratio}
Here we explore the relationship between the optimal masking ratio and the model size.
Table~\ref{tb:masking_ratio} reports the low-shot learning performance for various random masking ratios as we increase the model size.\footnote{Note that the performance of the ViT-S/16 can be improved by removing the Sinkhorn normalization, as we do in Table~\ref{tb:lowshot_imagenet}, however for consistency of evaluation with other models, we keep it in for this this ablation.}
\begin{table}[h]
    \centering
    \caption{{\bf Masking ratio.} Impact of pre-training random masking ratio (fraction of randomly dropped patches in each random mask) on ImageNet 1\% accuracy. Accuracy of larger models improves when leveraging aggressive masking during pre-training.}
    \label{tb:masking_ratio}
    \begin{tabular}{l c c c c}
        & \multicolumn{4}{c}{\bf Top 1} \\[2mm]
        & \multicolumn{4}{c}{Random Masking Ratio} \\
        \bf\small Architecture & 0.15 & 0.3 & 0.5 & 0.7 \\ \toprule
        ViT-S/16 & \cellcolor{fbApp}\bf 66.3 & 66.0 & 64.8 & -- \\
        ViT-B/16 & 68.8 & \cellcolor{fbApp}\bf 69.6 & -- & -- \\
        ViT-L/16 & \tt NaN & \tt NaN & \cellcolor{fbApp}\bf 70.1 & 69.4 \\
        \bottomrule
    \end{tabular}
\end{table}

When increasing the model size, we find that increasing the masking ratio (dropping more patches) is helpful for improving low-shot performance.
We also find that the ViT-L/16 runs with weak masking are unstable, while the runs with more aggressive masking are quite stable. However, we do not have sufficient evidence to claim that increasing the masking ratio always improves the stability of large ViT pre-training.

\paragraph{Augmentation Invariance and Low-Shot Learning}
We explore the importance of data-augmentation invariance for low-shot learning.
%By using different views for the anchor and target images, MSN learns representations that are invariant to the view-generating data-augmentations.
%We hypothesize that this invariance leads to improvements in low-shot learning. To explore this hypothesis, 
We pretrain a ViT-B/16 with MSN, where the teacher and anchor networks either share the input image view or use different input views; in both cases, the anchor view is always masked.
The views are constructed by applying random ColorJitter, Crop, Horizontal Flips, and GaussianBlur to the input image.
\begin{table}[h]
  \centering
  \caption{Impact of view-sharing during pre-training on low-shot accuracy (1\% of ImageNet-1K labels) of a ViT-B/16. The target view is constructed by applying random ColorJitter, Crop, Horizontal Flips, and GaussianBlur to the input image. When using the same image view, MSN finds a shortcut solution. Using color jitter prevents this pathological behaviour. Randomly applying additional geometric data transformations to the anchor further improves performance, demonstrating the importance of view invariance in the low-shot setting.}
  \label{tb:invariance_aug}
    \begin{tabular}{l | c}
    {\bf\small Anchor View Generation} & {\bf\small Top 1} \\\toprule
    Target View &  7.0 \\
    Target View + ColorJitter & 48.7  \\
    Target View + ColorJitter + Crop + Flip + GaussianBlur  & \cellcolor{fbApp}\bf 52.3 \\
    \bottomrule
    \end{tabular}
\end{table}

Table~\ref{tb:invariance_aug} reports top-1 accuracy when evaluating with 1\% of ImageNet-1K labels. Sharing the view leads to a top-1 accuracy of $7\%$; MSN finds a shortcut solution relying on color statistics. Using different colors in the input views resolves this pathological behaviour and achieves a top-1 of $48.3\%$. Further applying the geometric data-augmentations independently to the two views (as opposed to sharing views) further improves the performance to $52.3\%$, showing the importance of learning view-invariant representations in the low-shot setting.

% \im{Re-ordering ablations so compute stuff is the last one will ensure that the flow isn't broken. Every other ablation measures accuracy.}
\paragraph{Random Masking Compute and Memory}
We look at the effect of the random masking ratio, i.e., the fraction of dropped patches from the global anchor view, on the computational requirements of large model pre-training.
In each iteration we also generate 10 focal views (small crops) of each input image; the random masking ratio has no impact on these views.
\begin{table}[h]
    \centering
    \caption{Impact of random masking ratio on GPU memory usage and runtime when pre-training a ViT-L/7. Measurements are conducted on a single AWS {\tt p4d-24xlarge} machine, containing 8 A100 GPUs, using a batch-size of 2 images per GPU. In each iteration we also generate 10 focal views (small crops) of each input image; the random masking ratio has no impact on these views. Using more aggressive masking of the global view progressively reduces device memory utilization and speeds up training.}
    \label{tb:mask_compute}
    \begin{tabular}{c | c c}
        Masking Ratio & Mem./GPU & Throughput \\\toprule
        0.0 & 26G & 415 imgs/s \\
        0.3 & 21G & 480 imgs/s \\
        0.5 & 18G & 525 imgs/s \\
        \rowcolor{fbApp}
        0.7 & 17G & 600 imgs/s \\
        \bottomrule
    \end{tabular}
\end{table}

Table~\ref{tb:mask_compute} reports the memory consumption and throughput (imgs/s) of a ViT-L/7 model on a single AWS {\tt p4d-24xlarge} machine using a batch-size of 2 images per GPU.
As expected, using more aggressive masking of the global view progressively reduces device memory utilization and speeds up training.
For example, by randomly masking 70\% of the patches, we can use MSN to pre-train a full-precision ViT-Large with a patch-size of $7\times7$ on as few as 18 AWS {\tt p4d-24xlarge} machines.
Without masking, the same job requires over 42 machines when using the default batch-size of 1024 images.

\section{Conclusion}

We propose Masked Siamese Networks (MSNs), a self-supervised learning framework that leverages the idea of mask-denoising while avoiding pixel and token-level reconstruction. We demonstrate empirically that MSNs learn strong off-the-shelf representations that excel at label-efficient learning, while simultaneously improving the scalability of joint-embedding architectures.
% , while remaining competitive when more supervision is available for the downstream task.
By relying on view-invariant representation learning, MSN does require the specification of data transformations, and it may be that the optimal transformations and invariances are dataset and task dependant.
In future work, we plan to explore more flexible mechanisms to learn those transformations and also explore the use of equivariant representations.

% \paragraph*{Acknowledgments}
% Use unnumbered third level headings for the acknowledgments. All acknowledgments, including those to funding agencies, go at the end of the paper.

% \vfill\pagebreak
\bibliography{refs.bib}
\bibliographystyle{arxivtemplate/arxiv.bst}
\vfill\pagebreak
\appendix

\section{Implementations Details}
\label{apndx:implementation}

In this appendix section we provide the implementation details for MSN pre-training and evaluation.

\subsection{MSN Pre-training}

We adopt similar hyper-parameter settings that have previously been reported in the self-supervised literature for training Vision Transformers~\citep{caron2021emerging,chen2020exploring}.
Specifically, for pre-training, we use the AdamW optimizer~\citep{loshchilov2017decoupled} with a batch-size of 1024.
We linearly warm up the learning-rate from $0.0002$ to $0.001$ during the first 15 epochs, and decay it following a cosine schedule thereafter.
To construct the different image views, we apply the SimCLR data augmentations of~\citet{chen2020simple} to each sampled image; namely random crop, horizontal flip, color distortion, and Gaussian blur.
For each sampled image, we generate one large anchor view of size $224\times224$ pixels, and apply a random mask with a pre-specified masking ratio (0.15 for the ViT-S/16, 0.3 for the ViT-B/16 and ViT-B/8, and 0.7 for the ViT-L/7 and the ViT-B/4).
For each sampled image, we also generate 10 small focal anchor views of size $96\times96$ pixels.
We use a temperature of $0.1$ for the anchor network, and a temperature of $0.025$ for the target network.
Following the DINO method of~\citet{caron2021emerging}, we update the target network via an exponential moving average of the anchor network with a momentum value of $0.996$, and linearly increase this value to $1.0$ by the end of training.
Similarly, following~\citet{caron2021emerging}, weight decay is set to $0.04$ and increased to $0.4$ throughout training via a cosine schedule.
By default, we set the {\sc me-max} regularization weight $\lambda$ to $1.0$ and apply Sinkhorn normalization to the targets~\citep{caron2020unsupervised} to avoid having to tune the {\sc me-max} regularization weight; however, in general, we observe stronger MSN performance when omitting Sinkhorn normalization (see Appendix~\ref{apndx:additional_ablations}).
We train with a 3-layer projection head with output dimension 256 and batch-normalization at the input and hidden layers, and use 1024 prototypes of dimension 256.
We observe that using more prototypes has little effect on training, but using too few prototypes can hurt performance (see Appendix~\ref{apndx:additional_ablations}).
We discard the projection head during evaluation, and always use the representations computed from the output of the target encoder trunk for evaluation.

\subsection{Low-Shot Evaluation}

To avoid overfitting, we freeze the weights of the pre-trained model and train a linear classifier on top using 1, 2 or 5 labeled samples per class.
Specifically, we take a single center crop of each labeled image, extract its representation using the pre-trained model, and then train a classifier on these representations using L$_2$-regularized logistic regression.
Following~\citep{caron2021emerging}, we use the {\tt cyanure} package~\citep{mairal2019cyanure} to run logistic regression on the extracted representations.
This objective is smooth and strongly-convex (i.e., has a unique minimizer) and can therefore be efficiently solved for using the {\tt cyanure} python numerical solver on a single CPU core. All low-shot evaluations (including the 1\% ImageNet-1K evaluation) are computed with this procedure, except for models pre-trained using MAE~\citep{he2021masked}, which benefit from using partial fine-tuning~\citep{he2021masked}.

Partial fine-tuning corresponds to fine-tuning the last block of the pre-trained model along with a linear head. MAE benefits from partial fine-tuning, but for sufficiently large models, such as the ViT-H/14, this leads to significant overfitting in the low-shot regime. Our results in Table~\ref{tb:lowshot_imagenet} and Figure~\ref{fig:lowshot} report the best performance across evaluation methods for MAE. In particular, all the MAE results are obtained via partial fine-tuning, except for the 1 image per class setting, and all results with the ViT-H/14 architecture, which use a linear head. We compare both protocols in more detail in Appendix~\ref{apndx:additional_ablations}. 

\subsection{Linear Evaluation}
For linear evaluation, we use a similar procedure as~\citet{he2021masked}.
Specifically, we use a large batch-size of 16,384 images and train a linear classifier for 100 epochs using a learning rate of $6.4$, and decay it following a cosine schedule.
We only apply basic data augmentations; namely, random resized crops to a resolution of $224\times224$ pixels, and random horizontal flips.
We also L$_2$-normalize the representations before feeding them into the linear classifier, and optimize the classifier weights using SGD with Nesterov momentum.
We do not apply any weight-decay and do not use any warmup.

\subsection{Fine-Tuning Evaluation}
We follow the common practice for fine-tuning SSL pre-trained ViT models.
Specifically, we follow the setup of~\citep{touvron2021training,bao2021beit,he2021masked}.
We fine-tune a pre-trained ViT model for 100 epochs on the full supervised ImageNet-1K training data set using the AdamW~\citep{loshchilov2017decoupled} optimizer.
We use a batch size of 1024 with a learning rate of $0.002$.
The learning rate is linearly warmed-up during the first 5 epochs and decayed with a cosine schedule thereafter.
A layer-wise decay of $0.65$ is also applied, along with the data augmentations defined by RandAugment($9$, $0.5$) ~\citep{cubuk2019autoaugment}. We additionally use label smoothing set to $0.1$, mixup~\citep{zhang2017mixup} set to $0.8$, cutmix~\citep{yun2019cutmix} set to $1.0$, and drop path set to $0.2$.

\subsection{Transfer Learning}

\subsubsection{Linear Evaluation}
When performing linear evaluation for transfer learning, we freeze the weights of the ImageNet-1K pre-trained model and optimize a linear classifier on top. We resize each downstream image to $256 \times 256$ pixels, and take a single center crop of size $224 \times 224$ pixels. Next, we extract a representation of each image using the pre-trained model, and subsequently train a classifier on top using L$_2$-regularized logistic regression. %Similar to our low-shot evaluation, this objective is smooth and strongly-convex with a unique minimizer, and can therefore be solved for efficiently using the {\tt cyanure} python numerical solver on a single CPU core.

\subsubsection{Fine Tuning}
When performing end-to-end fine-tuning for transfer learning, we follow the protocol of DeiT and DINO~\citep{touvron2021training,caron2021emerging}.  Models transferred to CIFAR10 and CIFAR100 are fine-tuned for 1000 epochs using a batch size of 768 and a learning rate of $0.000075$. Models transferred to iNat18 and iNat19 models are fine-tuned for 300 epochs using a batch size of 1024 and a learning of $0.0001$.
All transfer fine-tuning experiments use the data augmentations defined by RandAugment($9$, $0.5$)~\citep{cubuk2019autoaugment}.
We also use label smoothing set to $0.1$, mixup~\citep{zhang2017mixup} set to $0.8$, cutmix~\citep{yun2019cutmix} set to $1.0$, and drop path set to $0.1$.
The learning rate is linearly warmed-up during the 5 first epochs and decayed with a cosine schedule thereafter.

\section{Theoretical Guarantees}
\label{apndx:theory}

In this section we describe how MSN pre-training provably avoids representation collapse.

Recall that in each iteration of pre-training, we sample a mini-batch of $B \geq 1$ images, and generate $M \geq 1$ anchor views of each image.
Here we show that MSN is guaranteed to avoid the trivial collapse of representations under the following assumption.

\begin{assumption}[Target Sharpening]
\label{ass:sharp}
The target $p^+$ is sharpened, such that it is not equal to the uniform distribution.
\end{assumption}
\begin{proposition}[Non-Collapsing Representations]
\label{prop:collapse}
Suppose Assumption~\ref{ass:sharp} holds.
If $f_\theta(\cdot)$ is such that the representations collapse, i.e., $z_{i,m} = z_{j,k}$ for all $i,j \in [B]$ and $m,k \in [M]$, then $\norm{\nabla_\theta H(p^+_{i}, p_{i,m})} + \norm{\nabla_\theta H(\overline{p})} > 0$ for all $i,m$.
\end{proposition}
\begin{proof}
For L$_2$-normalized representations and prototypes, the prediction $p_{i,m} \in \Delta_K$ corresponding to the $m^{\text{th}}$ view of the $i^{\text{th}}$ image in the mini-batch is given by
\[
    p_{i,m} \defeq \text{softmax}\left( \frac{z_{i,m} \cdot {\bf q}}{\tau} \right),
\]
where ${\bf q} \in \R^{K \times d}$ is the prototype matrix with $K > 1$ learnable prototypes, each of dimension $d$, and $\tau > 0$ is a scaler temperature.
Since $z_{i,m} = z_{j,k}$ for all $i,j \in [B]$ and $m,k \in [M]$, it holds that $z_{i,m} \cdot {\bf q} = z_{j,k} \cdot {\bf q}$, and therefore $p_{i,m} = p_{j,k}$.
Now consider two separate cases.

{\bf Case 1:} The predictions are equal to the uniform distribution, i.e., $p_{i,m} = \frac{1}{K} {\bf 1}_K$, where ${\bf 1_K} \in \R^K$ is the K-dimensional vector with each entry equal to $1$. In that case, since, by Assumption~\ref{ass:sharp}, the targets $p^+_{i}$ are sharpened such that they are not equal to the uniform distribution, it follows that $p_{i,m} \neq p^+_{i}$, and hence $\norm{\nabla_\theta H(p^+_{i},p_{i,m})} > 0$.

{\bf Case 2:} The predictions are not equal to the uniform distribution, i.e., $p_{i,m} \neq \frac{1}{K} {\bf 1}_K$. In that case, we have that the average prediction across all the anchor views $\overline{p} \defeq \frac{1}{MB}\sum^B_{i=1}\sum^M_{m=1} p_{i,m}$ is also not equal to the uniform distribution; i.e., $\overline{p} \neq \frac{1}{K}{\bf 1}_K$, and hence $\norm{\nabla_\theta H(\overline{p})} > 0$.
\end{proof}

Proposition~\ref{prop:collapse} provides a theoretical guarantee that MSN is immune to the trivial collapse of representations.
In short, the underlying principle is that entropy maximization encourages the anchor predictions to utilize the full set of prototypes, thereby preventing collapse to a non-uniform distribution, while target sharpening encourages the anchor predictions to be confident, thereby preventing collapse to the uniform distribution.

Note that the sharpening mechanism defined in Section~\ref{sec:methodology} (i.e., applying a temperature $\tau^+$ in the target network softmax) may not always satisfy Assumption~\ref{ass:sharp}, unless one introduces a simple tie-breaking rule.
In practice, such a rule is not necessary as the targets never become uniform (since we apply sharpening from the start of the training), although, it is important to use a sufficiently small temperature value in this case.

\section{Additional Ablations}
\label{apndx:additional_ablations}

\subsection{Sinkhorn Normalization}
By default, we set the {\sc me-max} regularization weight $\lambda$ to $1.0$ and apply Sinkhorn normalization on the targets to avoid having to tune the {\sc me-max} regularization weight.
However, we find that tuning the {\sc me-max} regularization weight and omitting Sinkhorn normalization can result in better performance; cf.~Table~\ref{tb:sinkhorn}.
\begin{table}[h]
    \centering
    \caption{{\bf Effect of Sinkhorn normalization.} We train a ViT-S/16 with a masking ratio of 0.15, and explore the impact of Sinkhorn normalization during pre-training on low-shot performance with 1\% of ImageNet-1K. Tuning the {\sc me-max} regularization weight and omitting Sinkhorn normalization gives better performance.}
    \label{tb:sinkhorn}
    \begin{tabular}{c c c | c}
        \bf Architecture & \bf Target Normalization & \bf{\sc me-max} weight $\lambda$ & \bf Top 1 \\\toprule
        \multirow{3}{*}{ViT-S/16} & Sinkhorn & 1.0 & 66.4 \\
        & None & 1.0 & 60.8 \\
        & None & 5.0 & \bf\cellcolor{fbApp} 67.2 \\
        \bottomrule
    \end{tabular}
\end{table}

\subsection{Number of Prototypes}
By default we train with 1024 prototypes of dimension 256. In this section we explore the effect of the number of prototypes on low-shot performance.
We observe that using more prototypes has little effect on training, but using too few prototypes can hurt performance; cf.~Table~\ref{tb:prototypes}.
\begin{table}[h]
    \centering
    \caption{{\bf Effect of number of prototypes.} We train a ViT-B/16 with a masking ratio of 0.3, and explore the impact of the number of prototypes during pre-training on low-shot performance with 1\% of ImageNet-1K. Using more prototypes has little effect on training, but using fewer prototypes can degrade performance.}
    \label{tb:prototypes}
    \begin{tabular}{c c | c}
        \bf Architecture & \bf Prototypes & \bf Top 1 \\\toprule
        \multirow{3}{*}{ViT-B/16} & 512 & 67.6 \\
        & 1024 & \bf\cellcolor{fbApp} 69.5 \\
        & 2048 & \bf\cellcolor{fbApp} 69.5 \\
        \bottomrule
    \end{tabular}
\end{table}

\subsection{Masked Auto-Encoder Partial Fine-Tuning}
Here we explore the low-shot performance of MAE when relying on alternative evaluation strategies.
\citet{he2021masked} conjecture that using pixel reconstruction in their MAE objective results in encoder representations of a lower semantic level than other methods, which may explain their difficulty in training a linear classifier on the frozen features.
In Table~\ref{tb:mae-eval} we explore the effect of partial fine-tuning on the low-shot performance of pre-trained MAE models.
Partial fine-tuning corresponds to fine-tuning the last block of the pre-trained model along with a linear head on the available labeled samples.
As observed in~\citep{he2021masked}, MAE benefits from partial fine-tuning.
However, for sufficiently large models, such as the ViT-H/14, this leads to significant overfitting in the low-shot regime, where one must instead resort to linear evaluation.
We report the best numbers for MAE across the two low-shot adaptation strategies in Figure~\ref{fig:lowshot}.
\begin{table}[h]
    \centering
    \caption{{\bf MAE low-shot evaluations.} Top-1 low-shot validation accuracy for different training strategies with MAE pre-trained models. Partial fine-tuning corresponds to fine-tuning the last block of the pre-trained model along with a linear head on the available labeled samples. Linear evaluation corresponds to training a linear classifier on top of the frozen pre-trained encoder. MAE benefits from partial fine-tuning, but for sufficiently large models, such as the ViT-H/14, this leads to significant overfitting in the low-shot regime, where one must instead one must resort to linear evaluation.}
    \label{tb:mae-eval}
    \begin{tabular}{l l c c c}
        & & \multicolumn{3}{c}{\bf\small Top 1}\\[1.5mm]
        & & \multicolumn{3}{c}{\small Images per Class}\\
        \bf\small Architecture & \bf\small Adaptation Strategy & 2 & 5 & $\sim$13 \\\toprule
        \multirow{2}{*}{ViT-B/16} & Partial Fine-Tuning & \bf\cellcolor{fbApp} 25.0 & \bf\cellcolor{fbApp} 40.5 & \bf\cellcolor{fbApp} 51.1 \\
        & Linear Eval. & 14.5 & 25.2 & 36.6 \\\midrule
        \multirow{2}{*}{ViT-L/16} & Partial Fine-Tuning & 19.3 & \bf\cellcolor{fbApp} 42.3 & \bf\cellcolor{fbApp} 59.4 \\
        & Linear Eval. & \bf\cellcolor{fbApp} 22.1 & 35.7 & 48.6 \\\midrule
        \multirow{2}{*}{ViT-H/14} & Partial Fine-Tuning & \tt rand & \tt rand & \tt rand \\
        & Linear Eval. & \bf\cellcolor{fbApp} 18.6 & \bf\cellcolor{fbApp} 32.8 & \bf\cellcolor{fbApp} 46.7 \\\bottomrule
    \end{tabular}
\end{table}

\section{MSN Representation Robustness}

Next we report the performance of MSN-pre-trained models on datasets that have been developed to evaluate the robustness of models trained on the standard ImageNet training set. We consider four datasets: ImageNet-A (\citet{hendrycks2021nae})\footnote{https://github.com/hendrycks/natural-adv-examples}, 
ImageNet-R (\citet{hendrycks2021many})\footnote{https://github.com/hendrycks/imagenet-r},
ImageNet-Sketch (\citet{wang2019learning})\footnote{https://github.com/HaohanWang/ImageNet-Sketch}, and
ImageNet-C (\citet{hendrycks2019robustness})\footnote{https://github.com/hendrycks/robustness}.

Table~\ref{tb:robustness} shows results for a ViT-B/16 pre-trained using MSN and fine-tuned using the protocol described in Appendix~\ref{apndx:implementation}. For comparison, we also report the performance of a fine-tuned ViT-B/16 pre-trained using MAE~\citep{he2021masked}, along with a supervised ResNet50 baseline, which is available in the PyTorch Torchvision package\footnote{https://github.com/pytorch/vision}. For ImageNet-A, -R, and -Sketch, we report top-1 accuracy on each provided validation set. For ImageNet-C, we use the mean Corruption Error metric proposed in~\citep{hendrycks2019robustness}, where values are normalized by AlexNet performance on the same validation set.
\begin{table}[h]
  \centering
  \caption{{\bf Evaluation on alternative ImageNet validation sets.} We evaluate the performance of a fine-tuned ViT-B/16 model on four alternative ImageNet validation sets: ImageNet-A, ImageNet-R, ImageNet-Sketch, and ImageNet-C. The metric used for the first three (-A, -R, and -Sketch) is top-1 accuracy on the validation set. On ImageNet-C, performance is measured in terms of mean Corruption Error (mCE)~\citep{hendrycks2019robustness}.}
  \label{tb:robustness}
       \begin{tabular}{l c c c c}
        & \bf IN-A & \bf IN-R & \bf IN-Sketch & \bf IN-C \\
        & (top-1 $\uparrow$) & (top-1 $\uparrow$)& (top-1 $\uparrow$) & (mCE $\downarrow$)\\ \midrule
        Supervised ResNet50 & 0.04 & 36.11  & 24.2 & 76.7 \\ 
        MAE ViT-B/16~\citep{he2021masked} & 35.9 & 48.3 & 34.5 & 51.7 \\ \midrule
        MSN ViT-B/16 & \bf\cellcolor{fbApp} 37.5 & \bf\cellcolor{fbApp} 50.0 & \bf\cellcolor{fbApp} 36.3 & \bf\cellcolor{fbApp} 46.6 \\ \bottomrule
    \end{tabular}
\end{table}

In each case we find that the performance of an MSN-pretrained ViT-B/16 is comparable or better than that of an MAE-pretrained ViT-B/16. Note also, that larger MAE-pretrained models achieve stronger performance on all four datasets~\citep{he2021masked}.

\section{MSN Invariance to Masking}
\label{apndx:representation _properties}
% In this appendix we take a closer looks at the properties of the representation learned by MSN.

%\subsection{Robustness to Masked Patches}
The goal of MSN pretraining is to denoise the input images at the representation level by ensuring that the representation of a masked input matches the representation of the unmasked one.
Here, we shows that MSN pretraining learns representations that are robust to patch masking.

In Table~\ref{tb:robust_mask}, we evaluate the performance of MSN and DINO when masking parts of an image during evaluation.
Models are evaluated on 1\% of ImageNet-1K using logistic regression on top of frozen features.
The logistic regression classifier is trained using masked images, and then evaluated on the standard ImageNet-1K validation set using unmasked images.

If the MSN representations are robust to missing image patches, then a linear classifier should be able to identify generalizable features when training on the representations of masked images.
On the other hand, if the representations output by the learned encoder are not robust to missing image patches, then a linear classifier would have difficulty finding generalizable features when training on the representations of masked images.
% If MSN is able to denoise the input image at the representation level, then a linear classifier should be able to identify generalizable features when training on the representations of masked images.
% On the other hand, if the learned encoder is not able to denoise the input images, then a linear classifier would have difficulty finding generalizable features when training on the representations of masked images.

We observe that masked pre-training results in representations that are more robust to patch removal, suggesting that MSN is performing an image denoising at the representation level.
Furthermore, models pre-trained with more aggressive masking exhibit this quality to a higher degree.
For example, the low-shot accuracy of ViT-L/7 pre-trained with aggressive masking is almost unaffected when we remove 70\% of the patches at test time; 75.1\% top-1 without dropping patches during evaluation versus 74.9\% top-1 when dropping 70\% of the patches during evaluation.
\begin{table}[h]
  \centering
  \caption{{\bf Robustness to missing patches (low-shot).} Evaluating the low-shot accuracy of pre-trained models on 1\% of ImageNet-1K when corrupting the annotated images by dropping patches. We train a linear classifier using masked images, and then evaluate on the standard ImageNet-1K validation set using unmasked images. We observe that MSN pre-training leads to representations that are more robust to masking. Moreover, models pre-trained with more aggressive masking exhibit this behaviour to a higher degree.}
  \label{tb:robust_mask}
   \begin{tabular}{l l c c c | c} 
        & & & \multicolumn{2}{c}{\bf Top 1} \\[2mm]
        & & & \multicolumn{2}{c}{Eval. Masking Ratio} \\
        \bf\small Alg. & \bf\small Arch. & \bf\small Pre-train Masking Ratio & 0.0 & 0.7 & $\Delta$ \\\toprule
        DINO & ViT-B/16 & 0.0 & 67.0 & 63.1 & -3.9\\\midrule
        \multirow{2}{*}{MSN} & ViT-B/16 & 0.3 & 69.5 & 67.1 & -2.4\\
         & ViT-L/7 & 0.7 & 75.1 & 74.9 & \cellcolor{fbApp}\bf -0.2\\ 
        \bottomrule
    \end{tabular}
\end{table}
\begin{table}[t]
  \footnotesize
  \centering
  \caption{{\bf Robustness to missing patches (cosine-similarity).} Average Cosine Distance between masked and unmasked representations of the same image. We compare the representations learned with MSN masked pre-training to those learned with DINO when using a ViT-B/16 encoder. The MSN ViT-B/16 is pre-trained with a masking ratio of 0.3. The cosine distances are computed and averaged over the ImageNet-1k validation set. The cosine similarity between masked and unmasked representations of the same image is higher when pre-training with MSN, supporting the observation that masked-pretraining results in representations that are more robust to patch-removal.\\}
  \label{tb:robust_mask_cosine}
        {\small
       \begin{tabular}{l c c c c c}
        & \multicolumn{5}{c}{\bf Cosine Similarity} \\[2mm]
        & \multicolumn{5}{c}{Eval. Masking Ratio} \\
        \bf\small Alg. & 0.15 & 0.3 & 0.5 & 0.7 & 0.9 \\ \toprule
        DINO & 0.98 & 0.97 & 0.92 & 0.81 & 0.56\\ 
        MSN & \cellcolor{fbApp}\bf 0.99 & \cellcolor{fbApp}\bf 0.99 & \cellcolor{fbApp}\bf 0.99 & \cellcolor{fbApp}\bf 0.98 & \cellcolor{fbApp}\bf 0.97\\ \bottomrule
    \end{tabular}}
\end{table}

We also report the average cosine distance between masked and unmasked representations of the same image in Table~\ref{tb:robust_mask_cosine}.
As expected, the cosine similarity between masked and unmasked representations of the same image is higher when pre-training with MSN, supporting the observation that masked-pretraining results in representations that are more robust to patch-removal.

\section{Qualitative Analysis}
\label{apndx:qualitative}
We qualitatively investigate the properties of the MSN pre-trained representations.
We follow the RCDM framework~\citep{bordes2021high} and train a conditional generative diffusion model, which maps a learned image representation back to pixel space. Specifically, RCDM takes as input random noise and the representation vector of an image computed by an SSL model (either an MSN pre-trained model or a DINO pre-trained model in this analysis), and aims to reconstruct the image as close as possible to the original one through a diffusion process. 
% This objective encourages the conditional generative diffusion model to extract as much information as possible from the image representation.

By using RCDM to sample an image based on its SSL representation, we can visualize how different pre-training strategies affect the degree of information contained in the representation.
Qualities that vary across RCDM samples represent information that is not contained in the pre-trained representation. Qualities that are semantically common across samples represent information contained in the representation.

\subsection{Comparison with DINO}
We apply RCDM on top of either a DINO or MSN pre-trained ViT-B/8 encoder to generate images of resolution $128\times 128$ pixels. RCDM is trained using unmasked images processed with the ViT-B/8 encoder. We then use masked images from the validation set at sampling time.

In Figure~\ref{fig:qualitiative_50percent}, we generate samples for RCDM when masking 50\% of the conditioning images.
The first column depicts images from the ImageNet validation set. 
The second column depicts the same image, but with 50\% of the patches masked. The representation of the masked image is used as conditioning for the RCDM diffusion model.
The subsequent columns in Figure~\ref{fig:qualitiative_50percent} show various images sampled from the conditioned RCDM diffusion model.  
We observe that the RCDM samples conditioned on the MSN representations (cf.~Figure~\ref{fig:qualitative_msn}) preserve the semantic category of the masked images, and remain visually close to the original image, despite the missing patches.
By contrast, the samples generated by the RCDM diffusion model conditioned on the DINO representations (cf.~Figure~\ref{fig:qualitative_dino}) are more blurry and do not preserve as well the semantic category of the masked images.

Figure~\ref{fig:qualitiative_80percent} depicts similar visualizations, but with 80\% of the patches masked.
In this case, even with 80\% of the patches missing, samples generated by RCDM conditioned on MSN representations preserve some of the structure in original images (cf.~Figure~\ref{fig:qualitative_msn_high_mask}).
On the other hand, conditioning on DINO representations leads to almost uniform background generation (cf.~Figure~\ref{fig:qualitative_dino_high_mask}). 
% Even with high masking value such 80\% of the input patches, see Figure \ref{fig:qualitative_msn_high_mask}, samples generated by a RCDM conditioned with MSN preserve some  structure of the original images. On the other hand using DINO representation in Figure \ref{fig:qualitative_dino_high_mask} leads to almost uniform background generation. %Figure\textcolor{red}{FIXME} shows different samples for a given input with different input mask and for MSN model trained with different random masking probability. We observe that MSN model trained with higher masking probability 1) is more robust to input mask 2) discard more information regarding the input. 

\subsection{MSN ViT-L/7 Visualizations}
We apply RCDM on top of the MSN pre-trained ViT-L/7 encoder to generate images with a resolution of $256 \times 256$ pixels. RCDM is trained using images with 70\% of patches masked. We then use masked images from the validation set (with various masking ratios) at sampling time, see Figures~\ref{fig:qualitiative_vitl7_100percent},~\ref{fig:qualitiative_vitl7_30percent}, and~\ref{fig:qualitiative_vitl7_10percent}.

Visualizations show that MSN  discards instance-specific information such as background, pose, and lighting, while retaining semantic information about the images, even when a large fraction of the patches are masked.

\begin{figure}[t]
    \begin{subfigure}{\linewidth}
        \centering
        \includegraphics[width=\linewidth]{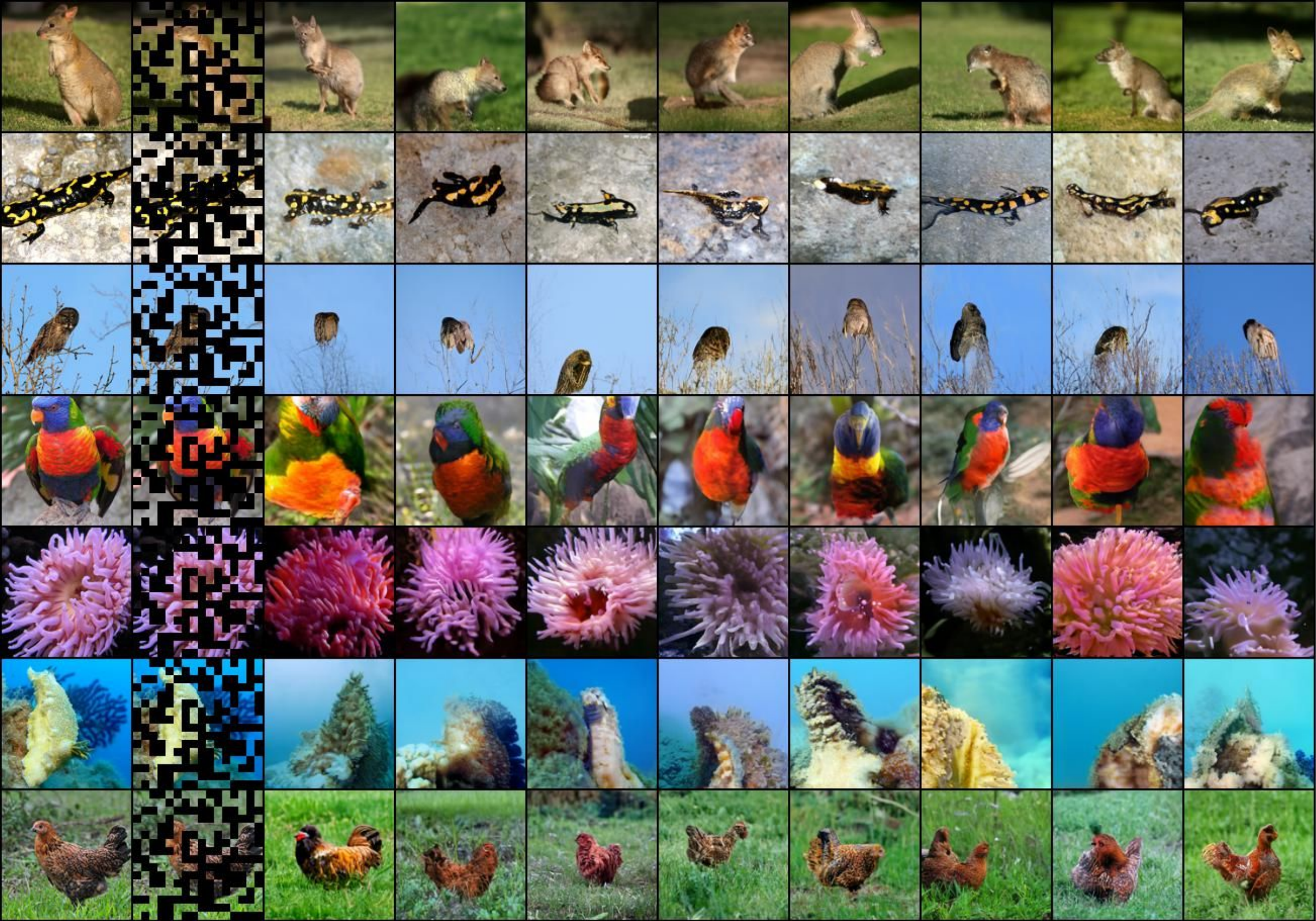}
        \caption{MSN Representations visualized on ImageNet validation set.}
        \label{fig:qualitative_msn}
    \end{subfigure}\vspace{1ex}
    \begin{subfigure}{\linewidth}
        \centering
        \includegraphics[width=\linewidth]{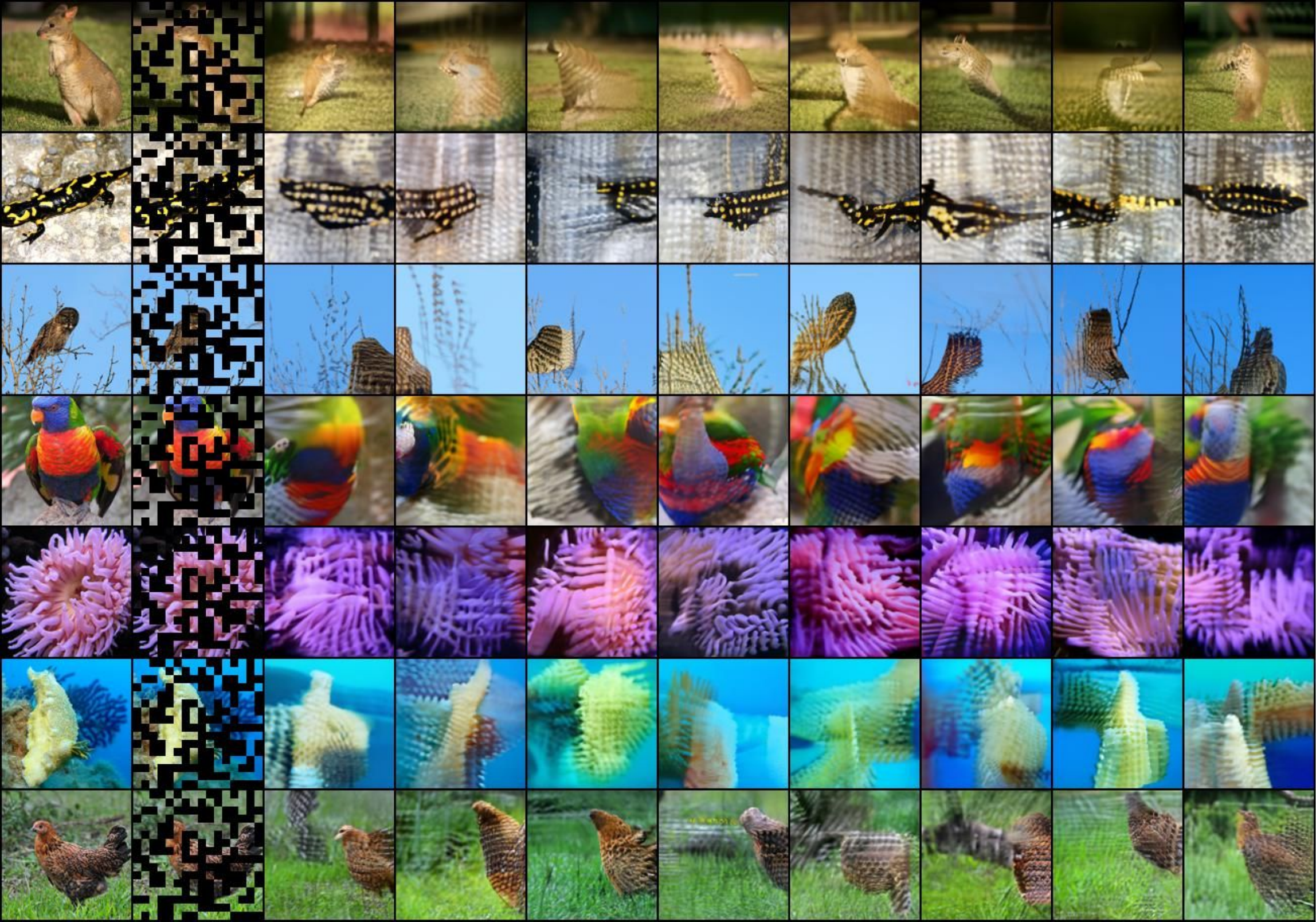}
        \caption{DINO Representations visualized on ImageNet validation set.}
        \label{fig:qualitative_dino}
    \end{subfigure}
    \caption{{\bf Visualizations of ViT-B/8 pre-trained representations computed from images with 50\% of patches masked.} First column: original image. Second column: image with 50\% of patches masked used to compute representations of an SSL pre-trained ViT-B/8 encoder. Other columns: RCDM sampling from generative model conditioned on SSL representation of masked image.}
    \label{fig:qualitiative_50percent}
\end{figure}
\begin{figure}[t]
    \begin{subfigure}{\linewidth}
        \centering
        \includegraphics[width=\linewidth]{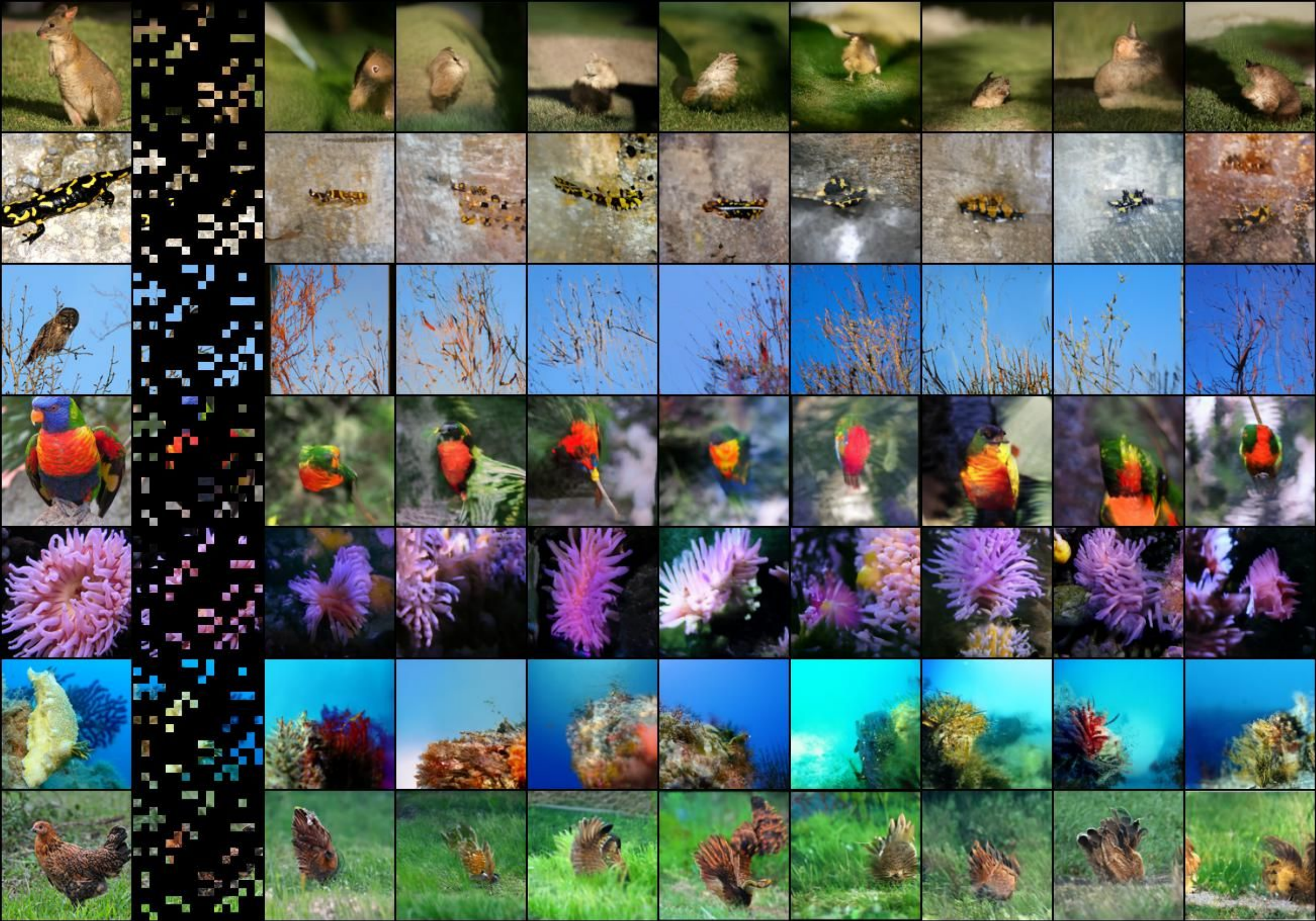}
        \caption{MSN representations visualized on ImageNet validation set.}
        \label{fig:qualitative_msn_high_mask}
    \end{subfigure}\vspace{1ex}
    \begin{subfigure}{\linewidth}
        \centering
        \includegraphics[width=\linewidth]{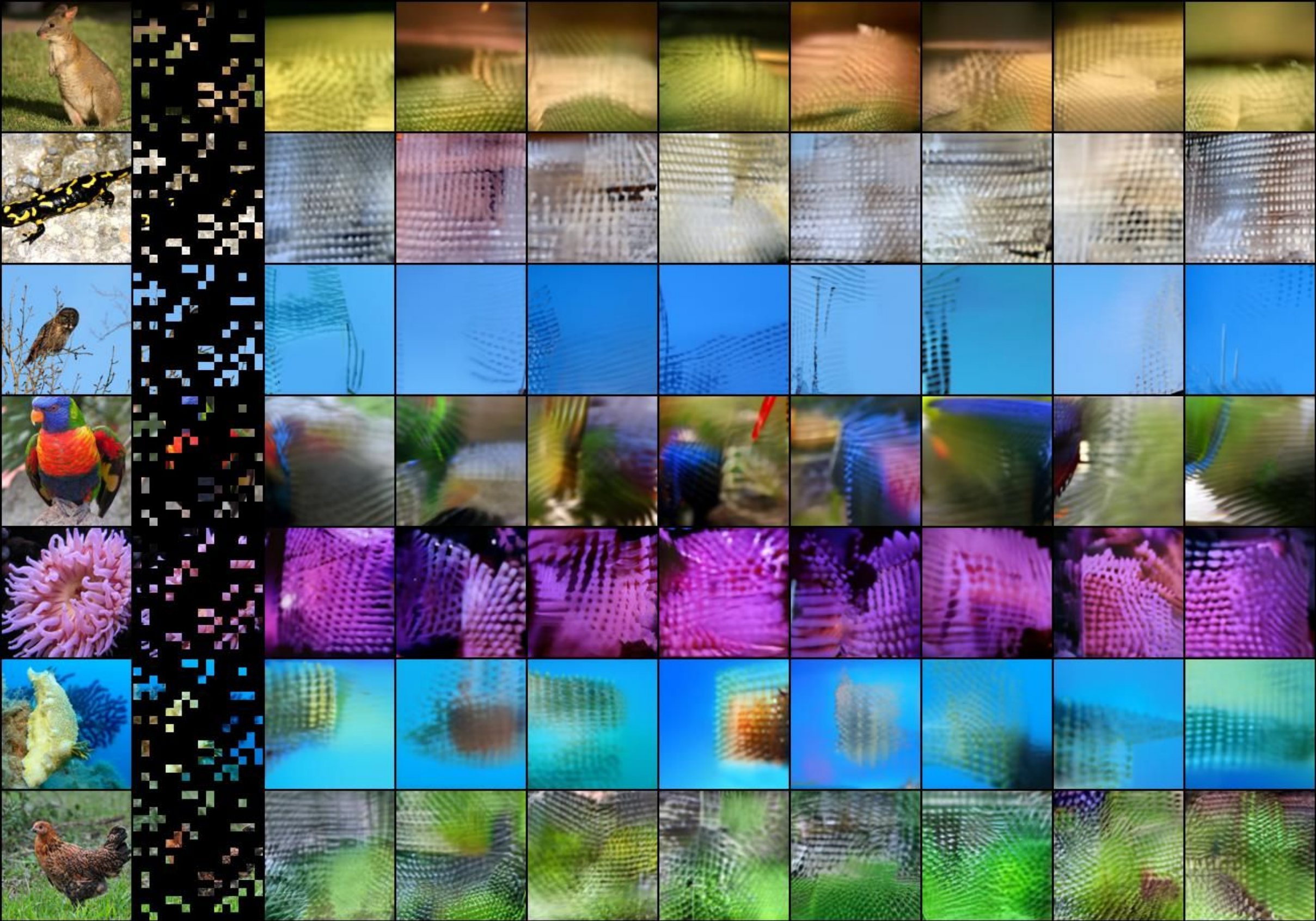}
        \caption{DINO representations visualized on ImageNet validation set.}
        \label{fig:qualitative_dino_high_mask}
    \end{subfigure}
    \caption{{\bf Visualizations of ViT-B/8 pre-trained representations computed from images with 80\% of patches masked.} First column: original image. Second column: image with 80\% of patches masked used to compute representations of an SSL pre-trained ViT-B/8 encoder. Other columns: RCDM sampling from generative model conditioned on SSL representation of masked image.}
    \label{fig:qualitiative_80percent}
\end{figure}

\begin{figure}[t]
    \centering
    \includegraphics[width=\linewidth]{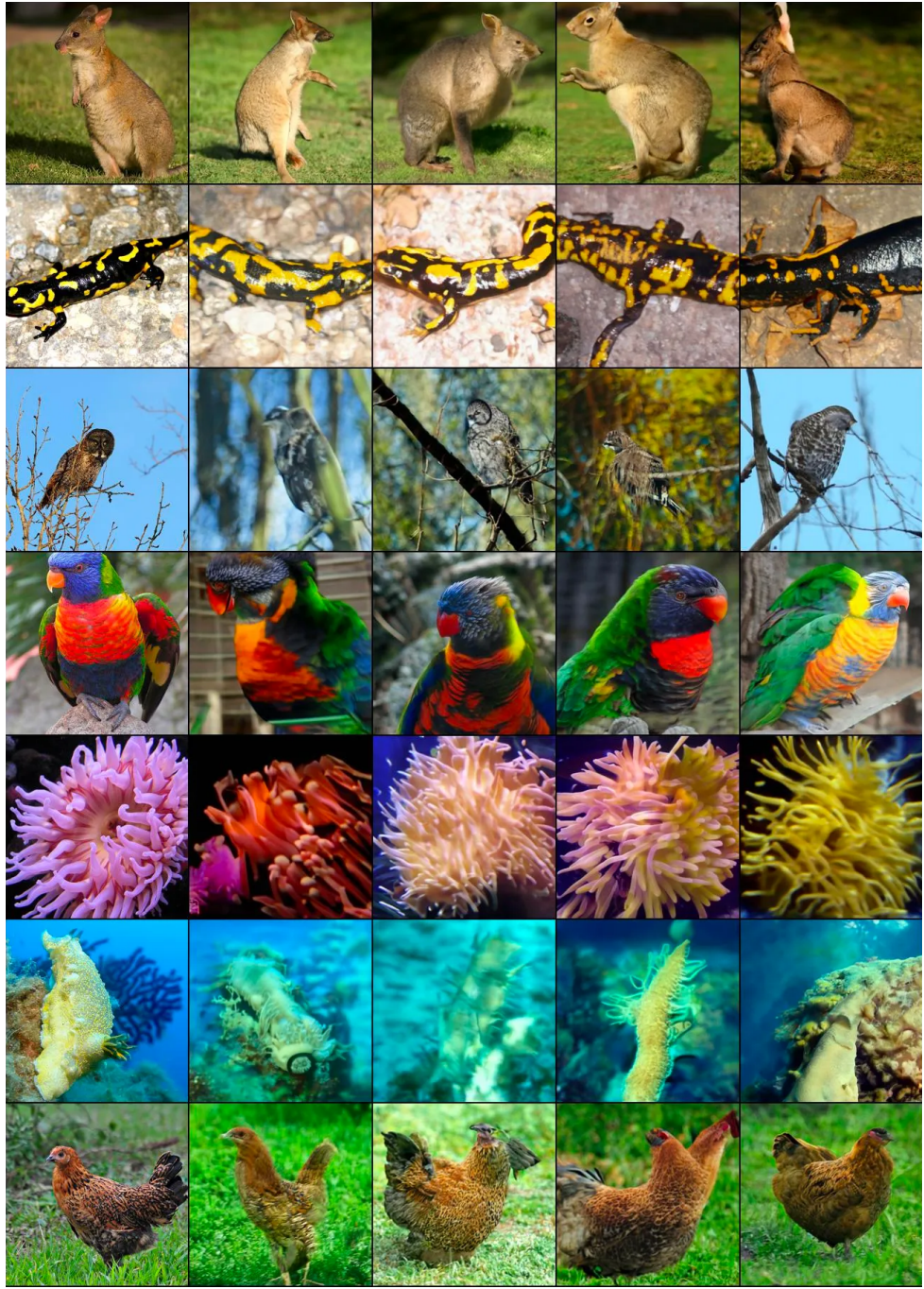}
    \caption{{\bf Visualizations of MSN pre-trained ViT-L/7 representations computed from unmasked images.} First column: original image. Other columns: RCDM sampling from generative model conditioned on MSN representation using a ViT-L/7 encoder. MSN representations are computed from unmasked images. Qualities that vary across samples represent information that the representation is invariant to; e.g., in this case, MSN discards background, pose, and lighting information. Qualities that are common across samples represent information contained in the pre-trained representation.}
    \label{fig:qualitiative_vitl7_100percent}
\end{figure}

\begin{figure}[t]
    \centering
    \includegraphics[width=\linewidth]{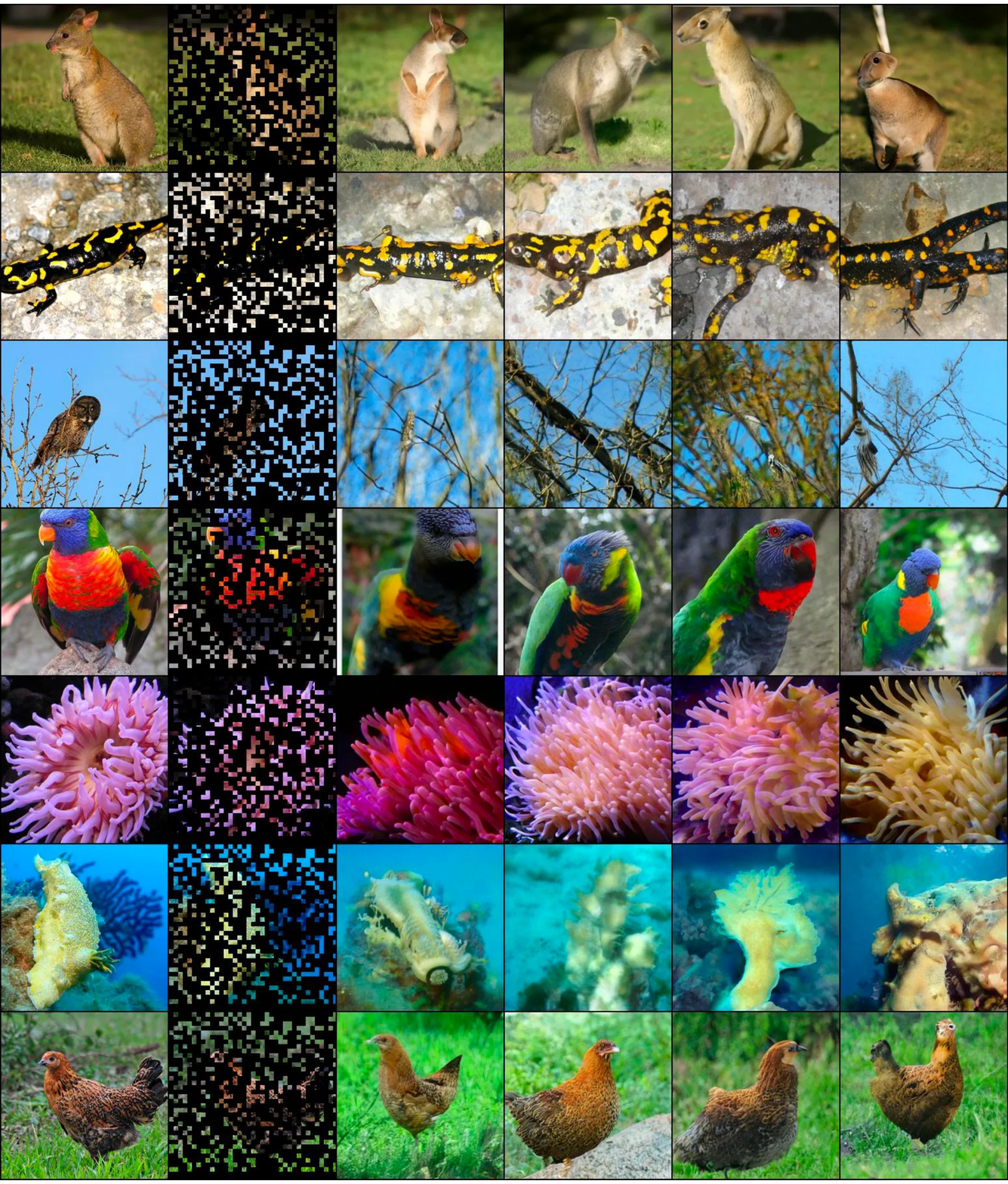}
    \caption{{\bf Visualizations of MSN pre-trained ViT-L/7 representations computed from images with 70\% of patches masked.} First column: original image. Second column: image with 70\% of patches masked used to compute representations of an SSL pre-trained ViT-L/7 encoder. Other columns: RCDM sampling from generative model conditioned on SSL representation of masked image. Qualities that vary across samples represent information that the representation is invariant to; e.g., in this case, MSN discards background, pose, and lighting information. Qualities that are common across samples represent information contained in the pre-trained representation.}
    \label{fig:qualitiative_vitl7_30percent}
\end{figure}

\begin{figure}[t]
    \centering
    \includegraphics[width=\linewidth]{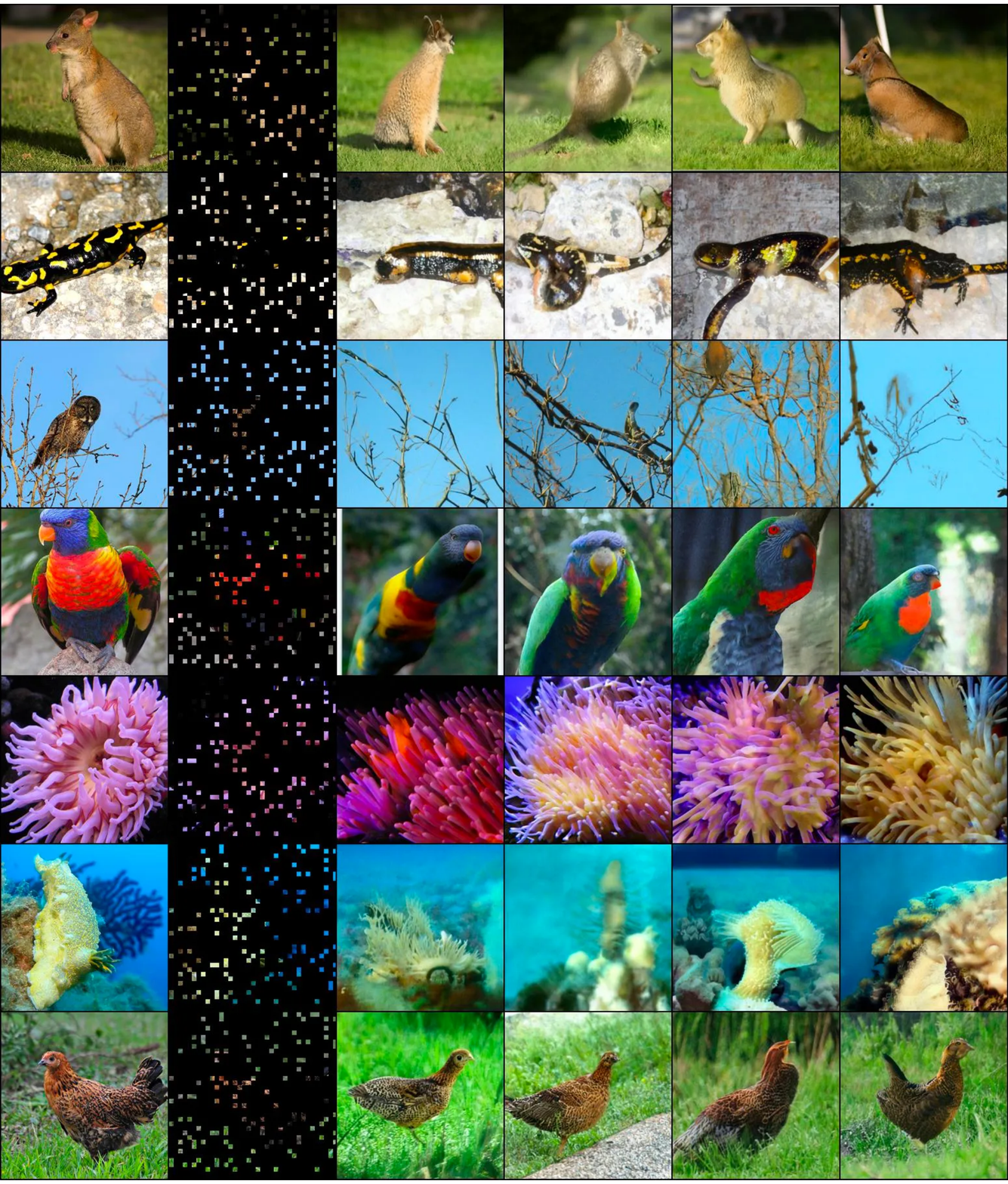}
    \caption{{\bf Visualizations of MSN pre-trained ViT-L/7 representations computed from images with 90\% of patches masked.} First column: original image. Second column: image with 90\% of patches masked used to compute representations of an SSL pre-trained ViT-L/7 encoder. Other columns: RCDM sampling from generative model conditioned on SSL representation of masked image. Qualities that vary across samples represent information that the representation is invariant to; e.g., in this case, MSN discards background, pose, and lighting information. Qualities that are common across samples represent information contained in the pre-trained representation. Even with high-masking ratio, MSN retains semantic information about the images.}
    \label{fig:qualitiative_vitl7_10percent}
\end{figure}

\end{document}